\documentclass[sn-mathphys,Numbered]{sn-jnl}

\usepackage{graphicx}%
\usepackage{multirow}%
\usepackage{amsmath,amssymb,amsfonts}%
\usepackage{amsthm}%
\usepackage{mathrsfs}%
\usepackage[title]{appendix}%
\usepackage{xcolor}%
\usepackage{textcomp}%
\usepackage{manyfoot}%
\usepackage{booktabs}%
\usepackage{algorithm}%
\usepackage{algorithmicx}%
\usepackage{algpseudocode}%
\usepackage{listings}%

\usepackage{times}
\usepackage{latexsym}
\usepackage{amsmath}
\usepackage{amssymb}
\usepackage{amsthm}
\usepackage{amsfonts}
\newtheorem{Lemma}{Lemma}[section]
\newtheorem{Theorem}{Theorem}[section]

\usepackage[T1]{fontenc}
\usepackage[utf8]{inputenc}
\usepackage{microtype}
\usepackage{inconsolata}
\usepackage{makecell}
\usepackage{booktabs}
\usepackage{threeparttable}
\usepackage{multirow}
\usepackage{multicol}
\usepackage{booktabs}
\usepackage{graphicx}
\usepackage{subfigure}
\usepackage{changes}
\usepackage{lscape}

\begin{document}

\title[CMPU]{Constraint Multi-class Positive and Unlabeled Learning for Distantly Supervised Named Entity Recognition}


\author[1]{\fnm{Yuzhe} \sur{Zhang}}\email{zyz2020@mail.ustc.edu.cn}

\author[2]{\fnm{Min} \sur{Cen}}\email{cenmin0127@mail.ustc.edu.cn}

\author*[1]{\fnm{Hong} \sur{Zhang}}\email{zhangh@ustc.edu.cn}

\affil*[1]{\orgdiv{School of Management}, \orgname{University of Science and Technology of China}, \orgaddress{\street{No. 96 Jinzhao Road}, \city{Hefei}, \postcode{230026}, \state{Anhui}, \country{China}}}

\affil[2]{\orgdiv{School of Data Science}, \orgname{University of Science and Technology of China}, \orgaddress{\street{No. 96 Jinzhao Road}, \city{Hefei}, \postcode{230026}, \state{Anhui}, \country{China}}}


\abstract{Distantly supervised named entity recognition (DS-NER) has been proposed to exploit the automatically labeled training data by external knowledge bases instead of human annotations. However, it tends to suffer from a high false negative rate due to the inherent incompleteness. To address this issue, we present a novel approach called \textbf{C}onstraint \textbf{M}ulti-class \textbf{P}ositive and \textbf{U}nlabeled Learning (CMPU), which introduces a constraint factor on the risk estimator of multiple positive classes. It suggests that the constraint non-negative risk estimator is more robust against overfitting than previous PU learning methods with limited positive data. Solid theoretical analysis on CMPU is provided to prove the validity of our approach. Extensive experiments on two benchmark datasets that were labeled using diverse external knowledge sources serve to demonstrate the superior performance of CMPU in comparison to existing DS-NER methods.}

\keywords{Named Entity Recognition, Distantly Supervised Learning, Positive and Unlabeled Learning, Theoretical Analysis}



\maketitle

\section{Introduction}
\label{s:intro}
Named Entity Recognition (NER) is a crucial task in information extraction and various downstream applications, aiming to identify entity mentions in text and assign them to predefined categories. However, the reliance on human annotation for creating large-scale labeled training datasets poses challenges to state-of-the-art supervised deep learning methods, particularly in specialized domains. To overcome this limitation, distant supervision techniques have been exploited to automatically generate annotated training data based on external knowledge sources, such as dictionaries and knowledge bases (KBs). Several studies \citep{luan2017scientific, gabor2018semeval, giorgi2019end} have explored distant supervision for NER.

\begin{figure}[htbp]
    \centering
    \includegraphics[scale=0.38]{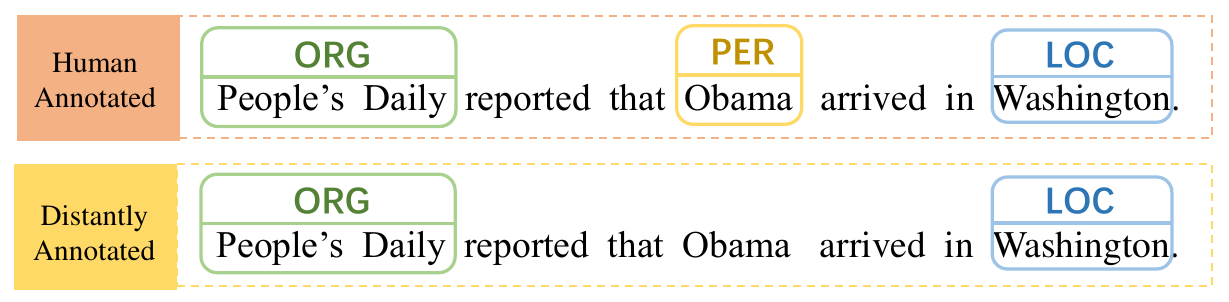}
    \caption{An annotated example with distant supervision. The entity "Obama" of type \texttt{PER} is not recognized.}
    \label{annotation}
\end{figure}

While previous distant supervision methods \citep{yang2018distantly, shang2018learning, mayhew2019named, cao2019low, peng2019distantly, liang2020bond, liu2021noisy, zhang2021biasing} have observed that distant labels obtained from dictionaries exhibit high precision, a significant incomplete labeling phenomenon persists. This arises from the fact that existing dictionaries and knowledge bases have limited coverage of entities. Consequently, simply considering unmatched tokens as negative samples leads to a high false negative rate (e.g., "Obama" is a false negative case in Figure \ref{annotation}) compared to human-annotated training data. This deficiency can misguide a supervised NER model to overfit false negative samples, adversely impacting its recall performance.

Recently, researchers have explored the application of binary Positive and Unlabeled (PU) learning to address the challenges in DS-NER tasks \citep{peng2019distantly}. PU learning is particularly suitable for handling distant supervision scenarios where external knowledge sources have limited coverage of positive samples, because it performs classification using a limited amount of labeled positive data and unlabeled data.

However, binary PU learning exhibits several drawbacks when applied to real-world DS-NER tasks. First, it relies on the one-vs-all strategy to convert the original multi-class classification problem into multiple binary classification problems. This approach becomes inefficient, particularly when dealing with numerous entity types. In such cases, training $N$ binary classifiers is required for an NER task with $N$ entity types. Second, the scale of predicted confidence values may vary among these binary classifiers, which can lead to challenges in ensuring mutually beneficial inferences for the final prediction.

Moreover, existing unbiased risk estimator of PU learning often faces the issue of overfitting when dealing with highly flexible models like deep neural networks, resulting in negative empirical risks. To address this issue, a non-negative risk estimator was proposed to maintain the risk calculated on unlabeled data from going negative \citep{kiryo2017positive}. However, this approach still struggles with overfitting as it fails to control the risk calculated on unlabeled data in time, leaving the classifier to suffer from overfitting.

To address these crucial challenges, we propose a novel method for PU learning known as \textbf{C}onstraint \textbf{M}ulti-class \textbf{P}ositive \textbf{U}nlabeled Learning (CMPU). Our approach includes a newly developed estimator for constraint non-negative risk, which not only ensures that the risk associated with unlabeled data remains non-negative but also maintains a dynamic balance between the risks of positive and unlabeled data. Our analysis indicates that when the risk of unlabeled data decreases rapidly, the classifier tends to overfit. Therefore, we introduce a constraint factor $\lambda$, to restrict the speed at which the risk of unlabeled data decreases. This helps maintain a dynamic equilibrium between the risks of positive and unlabeled data, thereby mitigating the overfitting problem.

In conclusion, our research makes the following key contributions:
\begin{itemize}
    \item We introduce CMPU, a novel and theoretically grounded approach for addressing the DS-NER task. CMPU incorporates solid theoretical analysis to enhance its practical applicability.
    \item Through extensive experiments, we demonstrate the effectiveness of CMPU's capability to mitigate the overfitting problem by balacing the positive and unlabeled risks.
    \item Our approach outperforms state-of-the-art DS-NER methods on benchmark datasets, highlighting its superior performance and effectiveness.
\end{itemize}

These contributions collectively contribute to advancing the field of DS-NER and provide valuable insights for addressing the challenges associated with incomplete annotations in distant supervision.

\section{Related Work}
\subsection{Distantly Supervised NER}
In recent years, there has been significant research attention on addressing the challenge of handling noisy labels, including false positives and false negatives, in the context of DS-NER \citep{yang2018distantly, shang2018learning, mayhew2019named, cao2019low, peng2019distantly, liang2020bond, liu2021noisy, zhang2021biasing}. Here we provide a concise overview of representative approaches in this field.

A line of work focuses on mitigating the impact of false negatives or incomplete labeling, which aligns with our research direction. To overcome the problem of incomplete labeling, some studies have employed partial annotation CRFs to consider all possible labels for unlabeled tokens \citep{shang2018learning, yang2018distantly}, but they still require a significant amount of annotated tokens or external tools. 
\cite{ni2017weakly} tackle the issue of label noise by implementing heuristic rules to eliminate sentences with low matching quality. Although this filtering strategy improved precision, it reduced recall.
\cite{shang2018learning} formulate AutoNER, a novel tagging scheme that identifies entity candidates by determining the connection status between adjacent tokens. The entity type is then determined for these candidates. To handle incomplete labeling, tokens with unknown tags are excluded from the loss calculation. 
\cite{mayhew2019named} introduce a constraint-driven iterative algorithm that detects false negatives in noisy data and assigns them lower weights. This leads to the creation of a weighted training set on which a weighted NER model is trained. 
\cite{peng2019distantly} propose AdaPU to prevent the model from overfitting to false negatives and perform a bounded non-negative positive-unlabeled learning approach. 
\cite{effland2021partially} make an assumption that the proportion of named entity tags compared to "O" tags remains relatively stable across different datasets with the same task. Based on this assumption, they suggest using a loss function called Expected Entity Ratio (EER) along with Partial Annotation Learning (PAL) loss for multi-task learning.

However, the application of binary PU learning in DS-NER is constrained by the underlying PU assumption and its limited efficiency. Our proposed method, CMPU, addresses these limitations and enables the wider utilization of PU learning in distant supervision scenarios. 

Another research direction aims to tackle both false positive and false negative annotation errors. \cite{cao2019low} propose a data selection scheme that computes scores for annotation confidence and annotation coverage. This scheme helps distinguish high-quality sentences from noisy ones, enabling the design of a name tagging model with two modules: sequence labeling and classification. The sequence labeling module focuses on high-quality portions, while the classification module targets the noisy portions. BOND \citep{liang2020bond} leverages the power of the pre-trained language models including BERT \citep{devlin-etal-2019-bert} and Roberta \citep{liu2019roberta}. It employs early stopping to prevent overfitting to noisy labels, resulting in an initialized model. Subsequently, BOND utilizes a teacher-student self-training framework to further enhance performance. \cite{liu2021noisy} propose a calibrated confidence estimation approach for DS-NER. They integrate this approach into an LSTM-CRF model within a self-training framework to mitigate the impact of noise. \cite{zhang2021biasing} explore the issue of noise in DS-NER from the perspective of dictionary bias. They begin by formulating DS-NER using a structural causal model and then identify the causes of false positives and false negatives. Then backdoor adjustment and causal invariance regularizer are employed to address the dictionary bias.

\subsection{PU Learning}
PU learning is a category of machine learning that trains a classifier using positive and unlabeled data \citep{elkan2008learning, du2014analysis}. While PU learning falls under the umbrella of semi-supervised learning, there is a fundamental difference between them: semi-supervised learning requires labeled negative data, whereas PU learning does not. In recent years, several advancements have enriched the theory of PU learning. \cite{kiryo2017positive} propose a non-negative risk estimator for PU learning, which enables the utilization of deep neural networks with limited labeled positive data. \cite{xu2017multi} introduce the concept of multi-positive and unlabeled learning with a margin maximization objective for multi-class classification problems. However, extending the margin maximization objective to popular deep learning architectures is not straightforward. \cite{hsieh2019classification} present a novel classification framework that incorporates biased negative data into PU learning, expanding the range of applications for PU learning.

\section{Methodology}
In this section, we introduce the proposed CMPU learning for DS-NER in the multi-class classification setting and provide solid theoretical analysis on CMPU.

\subsection{Learning Setup}
\label{setup}
Suppose that the data follow an unknown probability distribution $p(\mathbf{x}, y)$, where $\mathbf{x} \in \mathcal{X} \subset \mathbb{R}^{d}$ is the feature, and $\mathcal{Y}=\left\{0,1,2,\dots,C\right\}$ is the label space, where 0 refers to the negative class and $1,\dots,C$ refer to $C$ positive classes. The corresponding marginal distributions are $\mathbf{x} \sim p(\mathbf{x})$ and $P(y=c), c \in \mathcal{Y}$, respectively. 
Under the NER scenario, $\mathbf{x}$ represents the token embedding, and $y \in \left\{1,2,\dots,C\right\}$ means that this token belongs to certain entity category, while $y=0$ means that this token does not belong to any entity category.

Let $\pi_i=P(y=i), 1 \leq i \leq C$ be the prior of positive class $i$, and $\pi_0=P(y=0)$ be the prior of negative class, which satisfies $\sum_{i=0}^{C} \pi_i=1$. 
A decision function is denoted as $f \in \mathcal{H}: \mathcal{X} \rightarrow \mathcal{Y}$, where $\mathcal{H}$ is a function class. $f$ assigns a label to each token. 
The loss function is represented by $\ell(f(\mathbf{x}), y)$, which measures the difference between the predicted target $f(\mathbf{x})$ and ground-truth label $y$.

\textbf{Remark:} Here we  give an example to better illustrate the association between multi-class positive and unlabeled learning and NER task. 
Suppose that we are labeling a sentence 
$$\text{John Harris arrived at Texas Medical Center in Houston this afternoon},$$ 
and the ground-truth entities are $\text{Person}=\left\{\text{John Harris}\right\}$ and $\text{Location}=\left\{\text{Texas Medical Center, Houston}\right\}$. As a result, the entity types in this sentence are $\left\{\text{PER(Person)}, \text{LOC(Location)} \right\}$.
Following the learning setup in subsection \ref{setup}, we have $C=2$ and $\mathcal{Y}=\left\{0, 1, 2\right\}$.

The annotation format for NER tasks is BIO ($\textbf{B}$=beginning of an entity, $\textbf{I}$=inside an entity, and $\textbf{O}$=outside of an entity). The ground-truth label sequences under BIO annotation format is 
$$\text{B-PER \quad I-PER \quad O \quad O \quad B-LOC \quad I-LOC \quad I-LOC \quad O \quad B-LOC \quad O \quad O}.$$
For a distantly supervised NER setting, the label sequence may be
$$\text{B-PER \quad I-PER \quad O \quad O \quad O \quad O \quad O \quad O \quad B-LOC \quad O \quad O},$$
where entity "Texas Medical Center" of type LOC is missing because of the limited capacity of knowledge bases.
It is worth to mention that class "O" may either belong to the true negative (i.e., this token does not belong to any entity), or belong to certain positive class under the distantly supervised scenario.

It is worthy to mention that, NER tasks are commonly formulated as sequence labeling tasks in most researches. Sequence labeling tasks are common in natural language processing and speech recognition fields. Speech data can also be labeled in BIO format, which extends the applicability of our approach to other data structures.

\subsection{CMPU Learning}
The goal is to learn a decision function $f$ by minimizing the expected classification risk
\begin{equation}
    \label{mpn}
    R(f) = \sum_{i=1}^{C} \pi_i R_{P_i}^{+}(f) + \pi_0 R_{N}^{-}(f) = \sum_{i=1}^{C} \pi_i R_{P_i}^{+}(f) + \left(1 - \sum_{i=1}^{C} \pi_i\right) R_{N}^{-}(f). 
\end{equation}
where $R_{P_i}^{+}(f)=\mathbb{E}_{\mathbf{x} \sim p(\mathbf{x}|y=i)} \ell(f(\mathbf{x}), i)$ and $R_{N}^{-}(f)=\mathbb{E}_{\mathbf{x} \sim p(\mathbf{x}|y=0)} \ell(f(\mathbf{x}), 0)$ are the classification risk of the $i$th positive class and the negative class, respectively. We denote this classification risk as \textbf{MPN} (Multi-class Positive and Negative). Note that we do not have access to negative data in PU learning. The training instances are drawn from marginals of $p(\mathbf{x}, y)$ including positive instances $\mathcal{X}_{P_i}=\left\{\mathbf{x}_j^{P_i} \right\}_{j=1}^{n_{P_i}}, 1 \leq i \leq C$ and unlabeled instances $\mathcal{X}_{U}=\left\{\mathbf{x}_{j}^{U}\right\}_{j=1}^{n_{U}}$, where $n_{P_i}$ is the number of instances of the $i$th positive class, and $n_U$ is the number of unlabeled instances. Thus we cannot directly estimate Equation \eqref{mpn}. 

By the Law of Total Probability, we have
\begin{equation*}
    p(\mathbf{x}) = P(y=0)p(\mathbf{x}|y=0) + \sum_{i=1}^{C} P(y=i) p(\mathbf{x}|y=i),
\end{equation*}
which is equivalent to
\begin{equation*}
    \pi_0 p(\mathbf{x} | y=0) = p(\mathbf{x}) - \sum_{i=1}^{C} \pi_i p(\mathbf{x}|y=i).
\end{equation*}
Then we have
\begin{equation*}
    \begin{aligned}
        \left(1 - \sum_{i=1}^{C} \pi_i\right) R_{N}^{-}(f) &= \pi_0 R_{N}^{-}(f) \\ 
        &= \pi_0 \mathbb{E}_{\mathbf{x} \sim p(\mathbf{x}|y=0)} \ell(f(\mathbf{x}), 0) \\
        &= \int \pi_0 p(\mathbf{x}|y=0) \ell(f(\mathbf{x}), 0) d \mathbf{x} \\
        &= \int \left(p(\mathbf{x}) - \sum_{i=1}^{C} \pi_i p(\mathbf{x}|y=i) \right) \ell(f(\mathbf{x}), 0) d \mathbf{x} \\
        &= R^{-}(f) - \sum_{i=1}^{C} \pi_i R_{P_i}^{-}(f) \\
        &\overset{(i)}{=} R_{U}^{-}(f) - \sum_{i=1}^{C} \pi_i R_{P_i}^{-}(f),
    \end{aligned}
\end{equation*}
where $(i)$ is due to the basic assumption of PU learning that $p_{U}(\mathbf{x})=p(\mathbf{x})$.
Subsequently, we can transform Equation \eqref{mpn} into
\begin{equation}
    \label{mpu}
    R(f) = \sum_{i=1}^{C} \pi_i R_{P_i}^{+}(f) + R_{U}^{-}(f) - \sum_{i=1}^{C} \pi_i R_{P_i}^{-}(f),
\end{equation}
where $R_{U}^{-}(f)=\mathbb{E}_{\mathbf{x} \sim p_{U}(\mathbf{x})} \ell(f(\mathbf{x}), 0)$ and $R_{P_i}^{-}(f)=\mathbb{E}_{\mathbf{x} \sim p(\mathbf{x}|y=i)} \ell(f(\mathbf{x}), 0).$
We denote this classification risk as \textbf{MPU}(Multi-class Positive and Unlabeled), whose estimation requires the assumption that $p_{U}(\mathbf{x})=p(\mathbf{x})$ so that $\mathcal{X}_{U}$ can be used to estimate $R_{U}^{-}(f)$. It is easy to verify that MPN and MPU risks are equal in essence. For convenience in later discussion, we will collectively refer to MPN and MPU risks as $R(f)$. 

MPU can be na\"ively estimated by
\begin{equation}
    \label{mpu-estimator}
    \hat{R}_{MPU}(f) =\sum_{i=1}^{C} \pi_i \hat{R}_{P_i}^{+}(f) + \left(\hat{R}_{U}^{-}(f) - \sum_{i=1}^{C} \pi_i \hat{R}_{P_i}^{-}(f) \right),
\end{equation}
where $$\hat{R}^{+}_{P_i}(f)=\frac{1}{n_{P_i}} \sum_{j=1}^{n_{P_i}} \ell (f(\mathbf{x}_{j}^{P_i}), i),$$
$$\hat{R}^{-}_{P_i}(f)=\frac{1}{n_{P_i}} \sum_{j=1}^{n_{P_i}} \ell (f(\mathbf{x}_{j}^{P_i}), 0),$$
and 
$$\hat{R}_{U}^{-}(f)=\frac{1}{n_{U}} \sum_{j=1}^{n_{U}} \ell (f(\mathbf{x}_{j}^{U}), 0)$$ are the estimators of $R_{P_i}^{+}(f)$, $R_{P_i}^{-}(f)$, and $R_{U}^{-}(f)$, respectively.

When the second term of Equation \eqref{mpu-estimator} turns to negative, the model, such as flexible enough deep networks, always turns to overfit. \cite{kiryo2017positive} propose the non-negative risk estimator
\begin{equation}
    \label{mpu-nonnegative-estimator}
    \hat{R}_{MPU}(f) = \sum_{i=1}^{C} \pi_i \hat{R}_{P_i}^{+}(f) +  \max\left\{0, \hat{R}_{U}^{-}(f) - \sum_{i=1}^{C} \pi_i \hat{R}_{P_i}^{-}(f) \right\},
\end{equation}
which explicitly constrains the training risk of MPU to be non-negative and can mitigate overfitting to a certain extent. 

Although the MPU learning method is easy to understand, as previously stated, it faces two significant problems. Firstly, the constraint of non-negativity is merely a basic limitation. The MPU method only attempts to ensure the non-negative values of $\hat{R}_{U}(f) - \sum_{i=1}^{C} \pi_i \hat{R}_{P_i}^{-}(f)$, making it ineffective in controlling the risk $\hat{R}_{U}(f)$. As a result, the classifier has already experienced overfitting. If the trainable models, such as deep networks, exhibit a high level of flexibility, the risk associated with the unlabeled data $\hat{R}_{U}(f)$ would decrease rapidly. Consequently, the speed of overfitting would also rise. Secondly, the optimization approach of MPU suffers from a flaw, which focuses solely on maximizing the value of $\hat{R}_{U}(f) - \sum_{i=1}^{C} \pi_i \hat{R}_{P_i}^{-}(f)$ when it becomes negative. However, this approach leads to a classifier fitting well on unlabeled data but performing poorly on positive data. To address the aforementioned problems, we introduce a constraint MPU risk estimator called \textbf{CMPU}, which is defined as
\begin{equation}
    \label{cmpu-estimate}
        \hat{R}_{CMPU}(f) = \sum_{i=1}^{C} \pi_i \hat{R}_{P_i}^{+}(f) +  \max\left\{\lambda \sum_{i=1}^{C} \pi_i \hat{R}_{P_i}^{+}(f), \hat{R}_{U}^{-}(f) - \sum_{i=1}^{C} \pi_i \hat{R}_{P_i}^{-}(f) \right\},
\end{equation}
where $\lambda > 0$ is a constraint factor to avoid $\hat{R}_{U}^{-}(f)$ decline too quickly. 
The proposed $\hat{R}_{CMPU}(f)$ provides point-to-point remedial measures for the $\hat{R}_{MPU}(f)$. On the one hand, when 
$$\tau=\frac{\hat{R}_{U}^{-}(f) - \sum_{i=1}^{C} \pi_i \hat{R}_{P_i}^{-}(f)}{\sum_{i=1}^{C} \pi_i \hat{R}_{P_i}^{+}(f)} < \lambda,$$ CMPU prioritizes the risk of positive classes. The constraint of $\tau$ can successfully avoid $\hat{R}_{U}^{-}(f)$ decline too quickly. On the other hand, CMPU also controls the dynamic equilibrium between $\sum_{i=1}^{C} \pi_i \hat{R}_{P_i}^{+}(f)$ and $\hat{R}_{U}^{-}(f) - \sum_{i=1}^{C} \pi_i \hat{R}_{P_i}^{-}(f)$. In summary, these advantages make CMPU perform well on DS-NER task.

Following \cite{peng2019distantly}, we choose the mean absolute error (MAE) as the loss function for CMPU and other PU learning methods, instead of using common unbounded cross entropy loss, for the convenience of feasible training and theoretical analysis. Given label $\mathbf{y} \in \left\{0, 1\right\}^{C+1}$ in one-hot form, the loss on a token $\mathbf{x}$ is defined by
$$\ell(f(\mathbf{x}), \mathbf{y})=\frac{1}{C+1} \sum_{i=0}^{C} |y_i - 
f(\mathbf{x})_i|,$$
where $f(\mathbf{x}) \in \mathbb{R}^{C+1}$ is the softmax output, $y_i$ and $f(\mathbf{x})_i$ are the $i$th components of $\mathbf{y}$ and $f(\mathbf{x})$, respectively. It is easy to verify that $\ell(f(\mathbf{x}), \mathbf{y})$ is bounded between 0 and $\frac{2}{C+1}$.

\subsection{Theoretical Results}
\label{theory}
Let $\hat{f}_{CMPU}=\arg \min_{f \in \mathcal{H}} \hat{R}_{CMPU}(f)$, $f^{*}=\arg \min_{f \in \mathcal{H}} R(f)$. In this subsection, we establish the error bound of the consistency between the empirical risk $\hat{R}_{CMPU}(f)$ and the expected risk $R(f)$ in Theorem \ref{theo}, and the estimation error bound of $R(\hat{f}_{CMPU}) - R(f^{*})$ in Theorem \ref{theo2}. We relegate all proofs in Appendix.

\begin{Theorem}[Consistency]
    \label{theo}
    Assume that there exists a constant $\alpha$ such that $R_{U}^{-}(f) - \sum_{i=1}^{C} \pi_i R_{P_i}^{-}(f) - \lambda \sum_{i=1}^{C} \pi_i R^{+}_{P_i}(f) \geq \alpha$, and $\sup_{f \in \mathcal{H}} ||f||_{\infty} \leq C_f$ and $\sup_{|t| \leq C_f} \max_{0 \leq y \leq C} \ell(t, y) \leq C_{\ell}$, where $\mathcal{H}$ is a function class, $||f||_{\infty}$ is the infinity norm of function $f$, and $C_f > 0$, $C_{\ell} > 0$ are two positive numbers. For any $\delta > 0$, with probability at least $1 - \delta$, 
    \begin{equation*}
		|\hat{R}_{CMPU}(f) - R(f)|  
        \leq  C_{\delta} C_{\lambda} \left(\sum_{i=1}^{C} \frac{\pi_i}{\sqrt{n_{P_i}}} + \frac{1}{\sqrt{n_U}}\right) 
        +  \left((1 + \lambda) \sum_{i=1}^{C} \pi_i + 1 \right) C_{\ell} \Delta_{f},
    \end{equation*}
where 
\begin{equation*}
    \begin{aligned}
        C_{\delta}&=C_{\ell} \sqrt{\ln (2 / \delta) / 2}, \quad C_{\lambda}=\max\left\{2, 1 + \lambda \right\}, \\
        \Delta_{f}&=\exp\left(-\frac{2 (\alpha / C_{\ell})^2}{(1 + \lambda^2) \sum_{i=1}^{C} \pi_i^2 / n_{P_i} + 1 / n_{U}} \right).
    \end{aligned}
\end{equation*}
\end{Theorem}
Theorem \ref{theo} indicates that: (1) for fixed $f$, when $n_{P_i}, n_{U} \to \infty$, $\hat{R}_{CMPU}(f) \rightarrow R(f)$ in $\mathcal{O}_{p}(\sum_{i=1}^{C} \frac{\pi_i}{\sqrt{n_{P_i}}} + \frac{1}{\sqrt{n_{U}}})$; (2) the value of $\lambda$ should not be set too large, in order to obtain a fast convergence rate. 

\begin{Theorem}[Estimation error bound]
    \label{theo2}
    Assume that (a) $R_{U}^{-}(f) - \sum_{i=1}^{C} \pi_i R_{P_i}^{-}(f) - \lambda \sum_{i=1}^{C} \pi_i R^{+}_{P_i}(f) \geq \alpha$; (b) $\mathcal{H}$ is closed under negation, i.e., $f \in \mathcal{H}$ if and only if $-f \in \mathcal{H}$; (c) the loss function $\ell(t, y)$ is Lipschitz continuous in $t$ for all $|t| \leq C_{g}$ with a Lipschitz constant $L_{\ell}$, where $C_g$ is a positive number. For any $\delta > 0$, with probability at least $1 - \delta$,
    \begin{equation*}
        \begin{aligned}
            R(\hat{f}_{CMPU}) - R(f^{*}) \leq & (16 + 8 \lambda) L_{\ell} \sum_{i=1}^{C} \pi_i \mathfrak{R}_{n_{P_i}, p_{P_i}}(\mathcal{H}) + 8L_{\ell} \mathfrak{R}_{n_{U}, p}(\mathcal{H})  \\
            & + 2 C_{\delta}^{'} C_{\lambda} \left(\sum_{i=1}^{C} \frac{\pi_i}{\sqrt{n_{P_i}}} + \frac{1}{\sqrt{n_U}}\right) 
            + 2\left((1 + \lambda) \sum_{i=1}^{C} \pi_i + 1 \right) C_{\ell} \Delta_{f},
        \end{aligned}
    \end{equation*}
    where $C_{\delta}^{'}=C_{\ell} \sqrt{\ln(1/\delta) / 2}$, $C_{\lambda}, C_{\ell}$, and $\Delta_{f}$ are defined the same as those in Theorem \ref{theo}, and $\mathfrak{R}_{n_{P_i}, p_{P_i}}(\mathcal{H})$ and $\mathfrak{R}_{n_{U}, p}(\mathcal{H})$ are the Rademacher complexities of $\mathcal{H}$ for the sampling of size $n_{P_i}$ from $p_{P_i}(\mathbf{x})$ and of size $n_{U}$ from $p(\mathbf{x})$, respectively.
\end{Theorem}

\section{Experiments}
In this section, we evaluate the performance of the proposed CMPU and other standard approaches.
\subsection{Experimental Setup}

\noindent \textbf{Datasets} We analyze two benchmark NER datasets originating from distinct domains: (1) \textbf{BC5CDR}, sourced from the biomedical field, comprising 1,500 articles and encompassing 15,935 \texttt{Chemical} and 12,852 \texttt{Disease} entity mentions; (2) \textbf{CoNLL2003}, a renowned open-domain NER dataset, consisting of 1,393 English news articles and containing 10,059 \texttt{PER}, 10,645 \texttt{LOC}, 9,323 \texttt{ORG}, and 5,062 \texttt{MISC} entity mentions.

We acquire the subsequent distantly labeled datasets: (1) BC5CDR (Big Dict) is labeled using a dictionary \footnote{\url{https://github.com/shangjingbo1226/AutoNER}} provided in \cite{shang2018learning}; (2) BC5CDR (Small Dict) is labeled using a smaller dictionary created by selecting only the initial 20\% entries from the previous one; (3) CoNLL2003 (KB) \footnote{\url{https://github.com/cliang1453/BOND}} is labeled using the knowledge base Wikidata and made available in \cite{liang2020bond}; (4) CoNLL2003 (Dict) is labeled using an enhanced dictionary released by \cite{peng2019distantly} \footnote{\url{https://github.com/v-mipeng/LexiconNER}}. For dictionary labeling, we employ the strict string matching algorithm introduced in \cite{peng2019distantly}. The process of knowledge base labeling can be found in \cite{liang2020bond}.

Initially, we assess the distant annotation quality of the training data. Table \ref{app_table:stats} presents a comprehensive assessment of the distantly labeled training data, where the ground-truth data are human-annotated. The entity-level precision, recall, and $F_1$ scores are reported in Table \ref{app_table:stats}. The findings confirm the hypothesis from previous research that distant labels produced by dictionaries are typically of high precision but low recall.
\begin{table}[h]
    \caption{Entity-level evaluation results of the distantly annotated datasets based on human annotation.}
    \centering
    \begin{tabular}{llccc}
        \toprule
        \textbf{Datasets} & Entity Type & P. & R. & F1  \\
        \midrule
        \multirow{5}{*}{ \textbf{CoNLL03 (Dict)} } & PER &  91.49 & 72.80 & 81.08 \\
                                            & LOC &  96.87 & 31.60 & 47.65 \\
                                            & ORG & 87.01 & 55.86 & 68.04 \\
                                            & MISC & 95.74 & 48.37 & 64.27 \\
                                            & Overall & 91.62 & 52.15 & 66.47 \\
        \midrule
        \multirow{5}{*}{ \textbf{CoNLL03 (KB)} } & PER &  63.42 & 72.32 & 67.58 \\
                                            & LOC &  99.98 & 74.59 & 85.44 \\
                                            & ORG & 90.63 & 58.91 & 71.41 \\
                                            & MISC & 100.00 & 22.86 & 37.22 \\
                                            & Overall & 82.31 & 62.17 & 70.84 \\
        \midrule
        \multirow{3}{*}{ \textbf{BC5CDR (Big Dict)} } & Chemical &  95.10 & 76.17 & 84.59 \\
                                           & Disease &  81.20 & 55.48 & 65.92 \\
                                           & Overall & 89.45 & 66.95 & 76.58 \\
        \midrule
        \multirow{3}{*}{ \textbf{BC5CDR (Small Dict)} } & Chemical & 95.31 & 13.67 & 23.90 \\
                                           & Disease &  76.53 & 10.14 & 17.91 \\
                                           & Overall & 87.31 & 12.09 & 21.24 \\
        \bottomrule
    \end{tabular}
    \label{app_table:stats}
\end{table}

\noindent \textbf{Evaluation Metrics} Following previous works on studying NER tasks, all DS-NER techniques are trained using identical distantly labeled training data and assessed on the publicly available human-annotated test sets, considering strict entity-level micro precision, recall, and $F_1$ score, requiring both the entity type and boundary to exactly match the ground truth. Other than these three most popular metrics, we also report the token-level accuracy on test set. During the prediction phase, a continuous span with the same label is treated as a single entity.

Entity-level metrics can evaluate the performance of entity boundary identification better, but they ignore the model's ability to recognize other non-target entities. On the other hand, token-level metrics consider the accuracy of each token but may overlook the correctness of the entire entity and do not give enough consideration to the recognition of entity boundaries.
Therefore, when evaluating NER models, it is necessary to consider both entity-level and token-level metrics in order to fully assess the model's performance.

\noindent \textbf{Experimental Setup} 
In order to assess the effectiveness of DS-NER methods in practical usage within distantly supervised settings, we refrain from utilizing any human-annotated validation or test sets at any stage of the training process. The pretrained models are finetuned for 5 epochs on each dataset. Rather than reporting the performance of the best checkpoint, we present the results based on the final model. To reproduce the outcomes of AutoNER and BERT-ES, we adopt the provided code. For other methods, we report the results based on our own implementations. Pretrained \texttt{biobert-base-cased-v1.1} \footnote{\url{https://huggingface.co/dmis-lab/biobert-base-cased-v1.1}} and \texttt{bert-base-cased} \footnote{\url{https://www.huggingface.co/bert-base-cased}} are utilized for BC5CDR datasets and for CoNLL2003 datasets, respectively. The max input sequence length is restricted to 128 under all scenarios.

\subsection{Baselines}
\noindent \textbf{Fully-supervised methods.} We showcase the state-of-the-art (SOTA) performance achieved by fully supervised methods on the benchmark datasets BC5CDR \citep{wang2021improving} and CoNLL2003 \citep{wang2021automated}. The fully-supervised baselines include BERT, Maximum Entropy Model (MaxEnt), Conditional Random Fields (CRF), and BiLSTM-CRF. The reported results for the SOTA methods are obtained from their original papers. The performance of this group of models serves as upper-bound references.

\noindent \textbf{Distantly Supervised Methods.} We examine the following distantly supervised NER methods.
(1) KB Matching: This approach directly labels the test sets using dictionaries or knowledge bases.
(2) AutoNER \citep{shang2018learning}: Additional rules and dictionaries are applied to filter the distantly annotated datasets, and a novel tagging scheme is also introduced for the DS-NER task.
(3) BERT-ES \citep{liang2020bond}: This method employs early stopping to prevent BERT from overfitting to noisy distant labels.
(4) bnPU \citep{peng2019distantly}: bnPU applies binary PU learning risk estimation with mean absolute error (MAE) as the loss function to each entity type and infers the final types.
(5) MPU: This is the precursor of the proposed CMPU, which calculates the empirical risk using Equation \eqref{mpu-nonnegative-estimator}.
(6) Top-Neg \citep{xu-etal-2023-sampling}: This method employs the most similar negative instances, which exhibit high affinity with all positive samples for training.
(7) SANTA \citep{si-etal-2023-santa}: This approach introduces Memory-smoothed Focal Loss and Entity-aware KNN to relieve the entity ambiguity problem, together with a Boundary Mixup mechanism to enhance model resilience.

Please note that the full models in \cite{peng2019distantly} and \cite{liang2020bond} include self-training as post-processing steps, which are omitted in this evaluation. Our focus is primarily on assessing how well each model handles incomplete labeling issues in DS-NER tasks.

\subsection{Experimental Results}
\subsubsection{Main Results}
Table \ref{tab:main} displays the precision, recall, $F_1$ scores, and token-level accuracy for all considered methods on the test sets. CMPU outperforms most distantly supervised methods, with 12.5\% and 10.5\% increases of $F_1$ value and token-level accuracy across all of the four datasets, respectively, on average.
This demonstrates the superiority of our proposed risk estimator when trained on distantly-labeled data.
Although CMPU performs slightly poorer than the SOTA methods Top-Neg and SANTA in terms of $F_1$ values and token-level accuracy, it has its unique advantage of theoretical guarantees.

On the other hand, all PU learning-based methods show significantly higher recall on all datasets, indicating greater resilience to incomplete labeling. However, bnPU and MPU exhibit low precision on BC5CDR (Big Dict) despite performing well on BC5CDR (Small Dict). It is worth noting that although bnPU and MPU are based on the same probability principle, they may not necessarily perform similarly. bnPU is trained using the one-vs-all strategy, where the distribution of unlabeled data is different for each entity type, while MPU is simultaneously trained with all types keeping the same distribution of unlabeled data.

\begin{landscape}
\begin{table}[htbp]
    \caption{The entity-level metrics including precision ($\textbf{P.}$), recall ($\textbf{R.}$) and $F_1$ score ($\textbf{F}_1$), together with the token-level accuracy ($\textbf{Acc.}$) on test sets (in \%), where the bests are in bold. ``FS'', fully supervised; ``DS'', distantly supervised. The results under FS setting on BC5CDR (Big Dict) and BC5CDR (Small Dict), CoNLL2003 (KB) and CoNLL2003 (Big Dict) are identical because fully supervised methods learn with ground-truth labels instead of relying on the dictionary size.} 
    \centering
    \begin{tabular}{l@{\hspace{0.23cm}}lcccccccccccccccc}
    \toprule
    \multirow{2}{*}{\textbf{Mode}}  & \multirow{2}{*}{\textbf{Model}}  & \multicolumn{4}{c}{ \textbf{BC5CDR(Big Dict)}} & \multicolumn{4}{c}{ \textbf{BC5CDR(Small Dict)}}  & \multicolumn{4}{c}{ \textbf{CoNLL2003(KB)}}  & \multicolumn{4}{c}{ \textbf{CoNLL2003(Dict)}} \\ \cmidrule(lr){3-6} \cmidrule(lr){7-10} \cmidrule(lr){11-14}  \cmidrule(lr){15-18}
     & & $\textbf{P.}$ & $\textbf{R.}$ & $\textbf{F}_1$ & $\textbf{Acc.}$ & $\textbf{P.}$ & $\textbf{R.}$ & $\textbf{F}_1$ & $\textbf{Acc.}$ & $\textbf{P.}$ & $\textbf{R.}$ & $\textbf{F}_1$ & $\textbf{Acc.}$ & $\textbf{P.}$ & $\textbf{R.}$ & $\textbf{F}_1$ & $\textbf{Acc.}$\\
    \midrule
    \multirow{6}{*}{\textbf{FS}} & {Existing SOTA} & - & - & 90.99 & - & - & - & 90.99 & - & - & - & 94.60 & - & - & - & 94.60 & - \\
    & {BERT} & 79.75 & 88.46 & 83.88 & 95.33 & 79.75 & 88.46 & 83.88 & 95.33 & 88.00 & 90.08 & 89.03 & 97.18 & 88.00 & 90.08 & 89.03 & 97.18 \\
    & {MaxEnt} & 88.24 & 78.79 & 83.24 & 94.36 & 88.24 & 78.79 & 83.24 & 94.36 & 84.68 & 83.18 & 83.92 & 96.44 & 84.68 & 83.18 & 83.92 & 96.44 \\ 
    & {CRF} & 84.06 & 88.42 & 86.19 & 96.88 & 84.06 & 88.42 & 86.19 & 96.88 & 86.71 & 81.15 & 83.84 & 96.12 & 86.71 & 81.15 & 83.84 & 96.12 \\ 
    & {BiLSTM-CRF} & 89.21 & 84.45 & 86.76 & 97.55 & 89.21 & 84.45 & 86.76 & 97.55 & 88.53 & 90.21 & 89.36 & 97.98 & 88.53 & 90.21 & 89.36 & 97.98 \\ 
    \midrule
    \multirow{6}{*}{\textbf{DS}} &  KB Matching & 86.39 & 51.24 & 64.32 & 78.62 &  80.02 & 8.70 & 15.69  & 50.78 & 81.13 & 63.75 & 71.40 & 88.58 & 93.12 & 48.67 & 63.93 & 78.37 \\
    &AutoNER & 82.63 & 77.52 & 79.99  & 97.74 & 81.47 & 11.83 & 20.66  & 66.91 & 73.10 & 63.22 & 67.80 & 84.07 & 82.87 & 48.50 & 61.19 & 75.71 \\
    &BERT-ES & 80.43 & 67.94 & 73.66  & 89.78 & 75.60 & 9.71 & 17.21  & 54.82 & 81.38 & 64.80 & 72.15 & 90.59 & 85.77 & 50.63 & 63.68 & 78.64 \\ 
    &bnPU  &  48.12 & 77.06 & 59.24 & 72.08 &  64.93 & 76.43 & 70.21 & 89.30 & 70.74 & 78.91 & 74.61  & 93.66 & 69.90 & 74.67 & 72.21  & 88.98 \\
    &MPU  &  56.50 & 86.05 & 68.22  & 82.10 & 70.08 & 78.18 & 73.91 & 93.27 & 58.79 & 74.58 & 65.75  & 82.57 & 63.63 & 72.22 & 67.65  & 83.73 \\
    &Top-Neg  &  83.52 & 76.00 & 80.48  & 99.59 & 81.78 & 74.18 & 77.79 & 97.03 & 75.42 & 76.81 & 76.11  & 97.45 & 81.01 & 76.77 & 78.87  & 99.33 \\
    &SANTA  &  80.24 & 71.49 & 75.61 & 91.03 & 81.07 & 69.35 & 74.75 & 94.97 & 74.11 & 79.67 & 76.79 & 96.14 & 80.86 & 74.17 & 77.37  & 95.38 \\
    &\textbf{CMPU} & 72.09 & 87.63 & 79.10  & 96.25 & 75.00 & 75.76 & 75.38 & 95.59 & 71.45 & 80.98 & 75.92 & 95.26 & 79.43 & 77.46 & 78.43 & 96.64 \\
    \bottomrule
    \end{tabular}
    \label{tab:main}
\end{table}
\end{landscape}

\subsubsection{Impact of Constraint Factor $\lambda$}
We analyze the advantages of the constraint non-negative risk estimator in our study. We investigate the impact of the constraint factor $\lambda$ on the validation $F_1$ score, as shown in Figure \ref{impact_lambda}. The findings from Theorem \ref{theo} emphasize that it is not advisable to set a large value for $\lambda$, because CMPU has relatively low $F_1$ scores when $\lambda=2.0$. Interestingly, the optimal performance achieved at $\lambda=0.2$ for BC5CDR and $\lambda=0.1$ for CoNLL2003 highlights the significance of considering both situations comprehensively.

\begin{figure}[htbp]
    \centering
    \subfigure[]{
    \begin{minipage}[t]{0.4\linewidth}
        \centering
        \includegraphics[width=5cm, height=3cm]{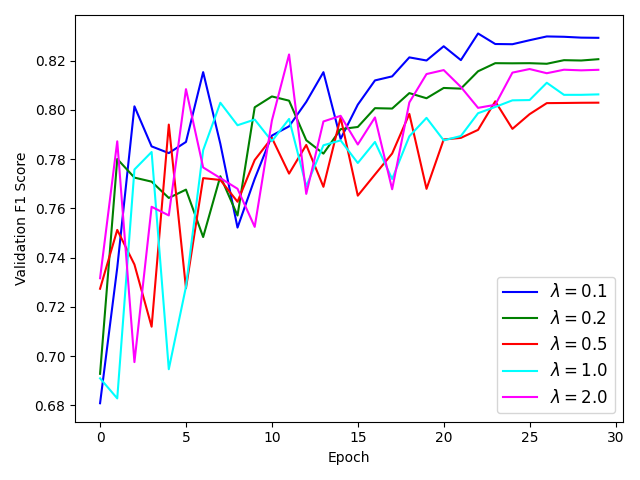}
    \end{minipage}
    }%
    \subfigure[]{
    \begin{minipage}[t]{0.4\linewidth}
        \centering
        \includegraphics[width=5cm, height=3cm]{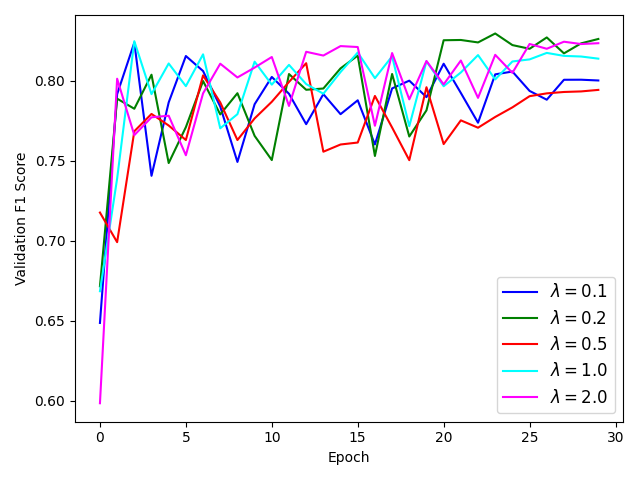}
    \end{minipage}
    }%
    \caption{$F_1$ score curves of different constraint factor $\lambda$ on (a) BC5CDR(Dict) and (b) CoNLL2003(Dict) datasets.}
    \label{impact_lambda}
\end{figure}

\subsubsection{Impact of Dictionary Size}

To analyze the estimation bias in bnPU and MPU resulting from the violation of the PU assumption, we conduct experiments using dictionaries with varying entity coverage. We consider the reference standard dictionaries used for labeling BC5CDR(Big Dict) and CoNLL2003(Dict) as benchmarks. For each dataset, we create a group of dictionaries by selecting the first 20$\%$, 40$\%$, 60$\%$, 80$\%$, and 100$\%$ of entries from the reference dictionaries. bnPU, MPU, and CMPU models are trained on the distantly labeled datasets generated by these dictionaries. The results based on BERT are presented in Figure \ref{dict_coverage} ((a)-(c) and (d)-(f) for the corresponding test sets).

A clear decreasing trend in precision can be observed for bnPU and MPU as the dictionary size increases, while CMPU has a slighter decrease than bnPU and MPU (Figure \ref{dict_coverage} (a)$\&$(d)). This phenomenon arises due to the violation of the PU assumption. As the dictionary coverage increases, the distribution of unlabeled data becomes more similar to the distribution of true negative data, rather than the overall data distribution. Consequently, the risk estimates of bnPU and MPU exhibit higher bias, resulting in lower precision. Although their recalls remain high, the $F_1$ scores still decrease. In contrast, the proposed CMPU effectively overcomes this limitation and achieves good performance across all dictionary sizes.

\begin{figure}[htbp]
    \centering
    \subfigbottomskip=1pt
    \subfigure[]{
    \begin{minipage}[t]{0.30\linewidth}
    \centering
    \includegraphics[width=1.5in]{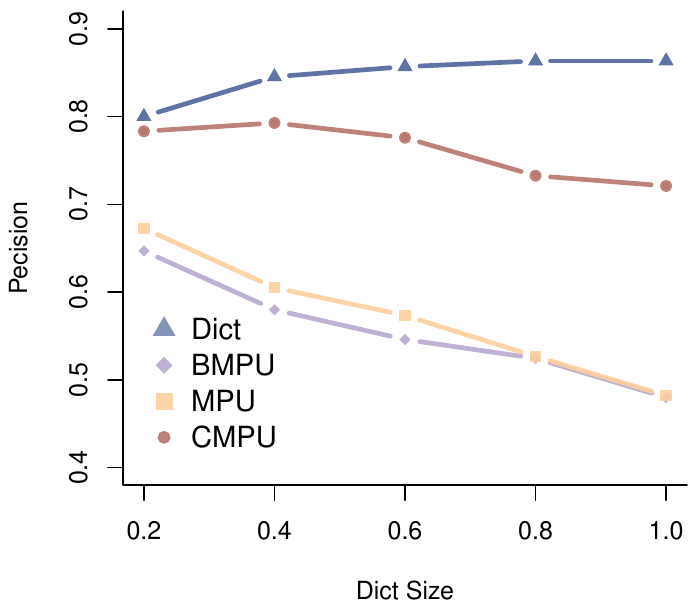}
    \end{minipage}
    }%
    \subfigure[]{
    \begin{minipage}[t]{0.30\linewidth}
    \centering
    \includegraphics[width=1.5in]{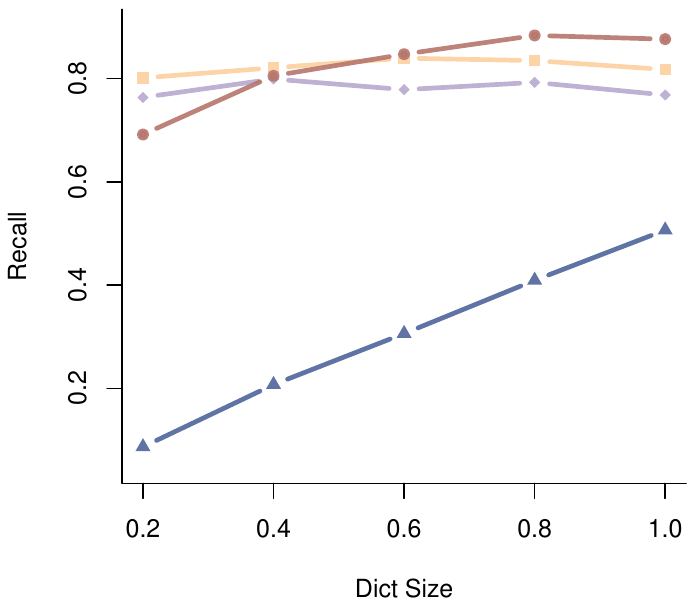}
    \end{minipage}
    }%
    \subfigure[]{
    \begin{minipage}[t]{0.30\linewidth}
    \centering
    \includegraphics[width=1.5in]{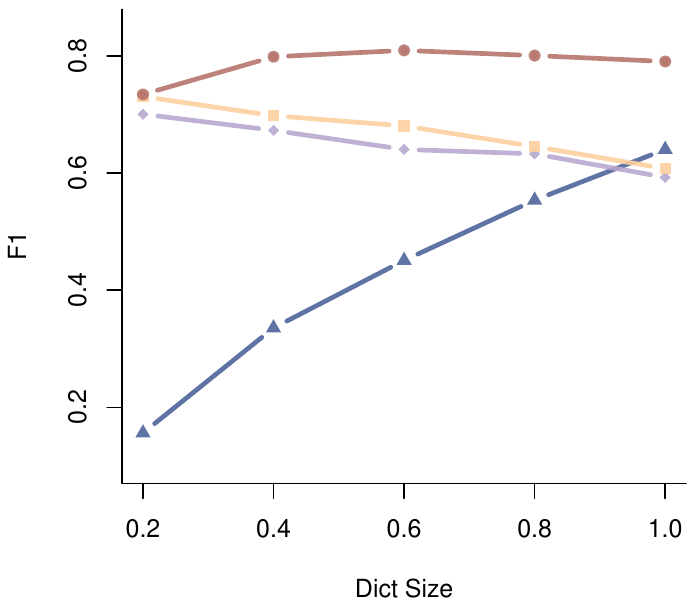}
    \end{minipage}
    }%
    \vspace{.00in}
    \subfigure[]{
    \begin{minipage}[t]{0.30\linewidth}
    \centering
    \includegraphics[width=1.5in]{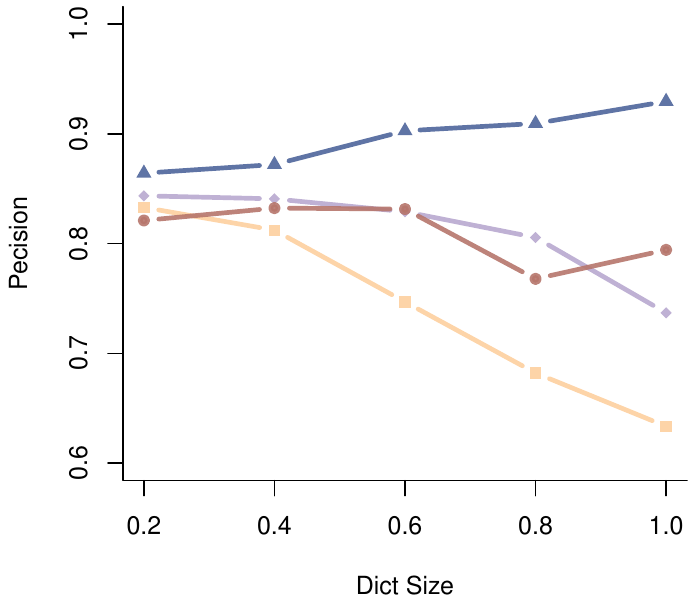}
    \end{minipage}
    }%
    \subfigure[]{
    \begin{minipage}[t]{0.30\linewidth}
    \centering
    \includegraphics[width=1.5in]{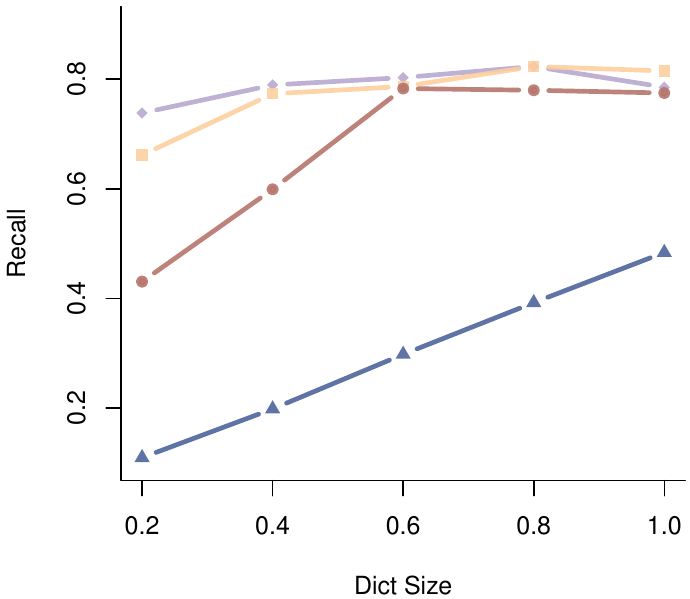}
    \end{minipage}
    }%
    \subfigure[]{
    \begin{minipage}[t]{0.30\linewidth}
    \centering
    \includegraphics[width=1.5in]{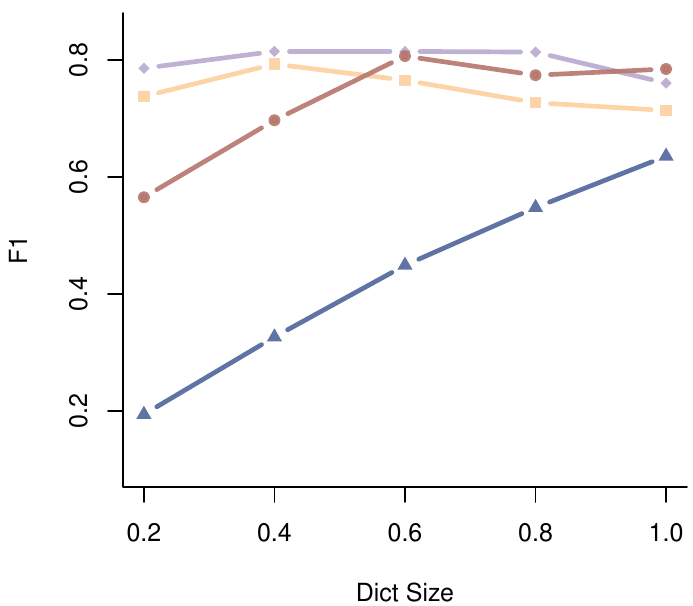}
    \end{minipage}
    }%
    \caption{Precision, recall and $F_1$ scores on BC5CDR ((a)-(c)) and CoNLL2003 ((d)-(f)) datasets under different dictionary sizes.}
    \label{dict_coverage}
\end{figure}

\section{Conclusion}
In this paper, we present CMPU, a novel constraint multi-class positive and unlabeled learning method, for the DS-NER task. CMPU introduces a constraint factor on the risk estimator of positive classes to avoid the fast decrease of PU risk. This approach effectively prevents the model from overfitting to false negatives. Through solid theoretical analysis and extensive empirical studies, we demonstrate the robustness of CMPU to various types of dictionaries and its ability to handle the problem of incomplete labeling. Our results highlight that CMPU significantly reduces the potential risk estimation bias caused by the PU assumption, making it a promising approach for DS-NER tasks.

With the rapid growth of text volumes and entity types, DS-NER methods face significant challenges. To be specific, there is a scarcity of annotated data for the vast majority of entity types. Emerging new entity types pose a continual challenge in creating annotated examples, particularly in valuable domains like biomedicine requiring specialized expertise. In this scenario, constructing a vast knowledge base for distant supervision annotation is extremely challenging.

In the era of large language models (LLMs) \citep{brown2020language, hoffmann2022training, touvron2023llama}, in-context learning \citep{dong2023survey} has become a new paradigm for natural language processing. With only a few task-specific examples as illustrations, LLMs can produce outcomes for a novel test input. Under the framework of in-context learning, LLMs have achieved promising results in named entity extraction \citep{brown2020language, chowdhery2023palm}. 
PromptNER \citep{ashok2023promptner} is a few-shot learning technique designed for NER, necessitating a collection of entity type definitions alongside annotated examples. When presented with a sentence, PromptNER prompts an LLM to generate a list of entities along with corresponding justifications substantiating their alignment with the provided entity type definitions.
GPT-NER \citep{wang2023gptner} transforms NER into a generation task and suggests a self-verification approach to correct the misclassification of NULL inputs as entities. 
\cite{xie2023selfimproving} utilizes a training-free self-improving framework, leveraging LLM to predict unlabeled corpora, thereby acquiring pseudo demonstrations to enhance LLM performance in zero-shot NER.

Despite the significant success of LLMs in NER tasks, we still emphasize the value of DS-NER methods due to the high computational resource requirements of LLMs. We leave the integration of DS-NER methods with LLMs as a future research direction.

\backmatter

\bmhead{Acknowledgments}

The work of HZ is partly supported by the National Natural Science Foundation of China (7209121, 12171451).


\begin{thebibliography}{10}

\bibitem{ashok2023promptner}
Dhananjay Ashok and Zachary~C. Lipton.
\newblock Promptner: Prompting for named entity recognition.
\newblock 2023.

\bibitem{bekker2020learning}
Jessa Bekker and Jesse Davis.
\newblock Learning from positive and unlabeled data: A survey.
\newblock {\em Machine Learning}, 109(4):719--760, 2020.

\bibitem{bishop2006pattern}
Christopher~M Bishop.
\newblock {\em Pattern recognition and machine learning}.
\newblock Springer, 2006.

\bibitem{brown2020language}
Tom Brown, Benjamin Mann, Nick Ryder, Melanie Subbiah, Jared~D Kaplan, Prafulla
  Dhariwal, Arvind Neelakantan, Pranav Shyam, Girish Sastry, Amanda Askell,
  et~al.
\newblock Language models are few-shot learners.
\newblock {\em Advances in neural information processing systems},
  33:1877--1901, 2020.

\bibitem{cao2019low}
Yixin Cao, Zikun Hu, Tat-seng Chua, Zhiyuan Liu, and Heng Ji.
\newblock Low-resource name tagging learned with weakly labeled data.
\newblock In {\em Proceedings of the 2019 Conference on Empirical Methods in
  Natural Language Processing and the 9th International Joint Conference on
  Natural Language Processing}, pages 261--270, 2019.

\bibitem{chowdhery2023palm}
Aakanksha Chowdhery, Sharan Narang, Jacob Devlin, Maarten Bosma, Gaurav Mishra,
  Adam Roberts, Paul Barham, Hyung~Won Chung, Charles Sutton, Sebastian
  Gehrmann, et~al.
\newblock Palm: Scaling language modeling with pathways.
\newblock {\em Journal of Machine Learning Research}, 24(240):1--113, 2023.

\bibitem{devlin-etal-2019-bert}
Jacob Devlin, Ming-Wei Chang, Kenton Lee, and Kristina Toutanova.
\newblock {BERT}: Pre-training of deep bidirectional transformers for language
  understanding.
\newblock In {\em Proceedings of the 2019 Conference of the North {A}merican
  Chapter of the Association for Computational Linguistics}, pages 4171--4186,
  2019.

\bibitem{dong2023survey}
Qingxiu Dong, Lei Li, Damai Dai, Ce~Zheng, Zhiyong Wu, Baobao Chang, Xu~Sun,
  Jingjing Xu, Lei Li, and Zhifang Sui.
\newblock A survey on in-context learning.
\newblock 2023.

\bibitem{du2014analysis}
Marthinus~C Du~Plessis, Gang Niu, and Masashi Sugiyama.
\newblock Analysis of learning from positive and unlabeled data.
\newblock In {\em Advances in neural information processing systems},
  volume~27, pages 703--711, 2014.

\bibitem{effland2021partially}
Thomas Effland and Michael Collins.
\newblock Partially supervised named entity recognition via the expected entity
  ratio loss.
\newblock {\em Transactions of the Association for Computational Linguistics},
  9:1320--1335, 2021.

\bibitem{elkan2008learning}
Charles Elkan and Keith Noto.
\newblock Learning classifiers from only positive and unlabeled data.
\newblock In {\em {ACM SIGKDD International Conference on Knowledge Discovery
  and Data Mining}}, pages 213--220, 2008.

\bibitem{gabor2018semeval}
Kata G{\'a}bor, Davide Buscaldi, Anne-Kathrin Schumann, Behrang QasemiZadeh,
  Haifa Zargayouna, and Thierry Charnois.
\newblock Semeval-2018 task 7: Semantic relation extraction and classification
  in scientific papers.
\newblock In {\em Proceedings of The 12th International Workshop on Semantic
  Evaluation}, pages 679--688, 2018.

\bibitem{giorgi2019end}
John Giorgi, Xindi Wang, Nicola Sahar, Won~Young Shin, Gary~D Bader, and
  Bo~Wang.
\newblock End-to-end named entity recognition and relation extraction using
  pre-trained language models.
\newblock {\em arXiv preprint arXiv:1912.13415}, 2019.

\bibitem{hoffmann2022training}
Jordan Hoffmann, Sebastian Borgeaud, Arthur Mensch, Elena Buchatskaya, Trevor
  Cai, Eliza Rutherford, Diego de~Las Casas, Lisa~Anne Hendricks, Johannes
  Welbl, Aidan Clark, et~al.
\newblock Training compute-optimal large language models.
\newblock {\em arXiv preprint arXiv:2203.15556}, 2022.

\bibitem{hsieh2019classification}
Yu-Guan Hsieh, Gang Niu, and Masashi Sugiyama.
\newblock Classification from positive, unlabeled and biased negative data.
\newblock In {\em International Conference on Machine Learning}, pages
  2820--2829, 2019.

\bibitem{ju2020pumad}
Hyunjun Ju, Dongha Lee, Junyoung Hwang, Junghyun Namkung, and Hwanjo Yu.
\newblock Pumad: Pu metric learning for anomaly detection.
\newblock {\em Information Sciences}, 523:167--183, 2020.

\bibitem{kiryo2017positive}
Ryuichi Kiryo, Gang Niu, Marthinus~C Du~Plessis, and Masashi Sugiyama.
\newblock Positive-unlabeled learning with non-negative risk estimator.
\newblock {\em Advances in neural information processing systems}, 30, 2017.

\bibitem{ledoux1991probability}
Michel Ledoux and Michel Talagrand.
\newblock {\em Probability in Banach Spaces: isoperimetry and processes},
  volume~23.
\newblock Springer Science \& Business Media, 1991.

\bibitem{li2020unified}
Xiaoya Li, Jingrong Feng, Yuxian Meng, Qinghong Han, Fei Wu, and Jiwei Li.
\newblock A unified mrc framework for named entity recognition.
\newblock In {\em Proceedings of the 58th Annual Meeting of the Association for
  Computational Linguistics}, pages 5849--5859, 2020.

\bibitem{liang2020bond}
Chen Liang, Yue Yu, Haoming Jiang, Siawpeng Er, Ruijia Wang, Tuo Zhao, and Chao
  Zhang.
\newblock Bond: Bert-assisted open-domain named entity recognition with distant
  supervision.
\newblock In {\em ACM SIGKDD International Conference on Knowledge Discovery
  and Data Mining}, 2020.

\bibitem{liu2021noisy}
Kun Liu, Yao Fu, Chuanqi Tan, Mosha Chen, Ningyu Zhang, Songfang Huang, and
  Sheng Gao.
\newblock Noisy-labeled ner with confidence estimation.
\newblock In {\em Proceedings of the 2021 Conference of the North American
  Chapter of the Association for Computational Linguistics: Human Language
  Technologies}, pages 3437--3445, 2021.

\bibitem{liu2019roberta}
Yinhan Liu, Myle Ott, Naman Goyal, Jingfei Du, Mandar Joshi, Danqi Chen, Omer
  Levy, Mike Lewis, Luke Zettlemoyer, and Veselin Stoyanov.
\newblock Roberta: A robustly optimized bert pretraining approach.
\newblock {\em arXiv preprint arXiv:1907.11692}, 2019.

\bibitem{luan2017scientific}
Yi~Luan, Mari Ostendorf, and Hannaneh Hajishirzi.
\newblock Scientific information extraction with semi-supervised neural
  tagging.
\newblock In {\em Proceedings of the 2017 Conference on Empirical Methods in
  Natural Language Processing}, pages 2641--2651, 2017.

\bibitem{mayhew2019named}
Stephen Mayhew, Snigdha Chaturvedi, Chen-Tse Tsai, and Dan Roth.
\newblock Named entity recognition with partially annotated training data.
\newblock In {\em Proceedings of the 23rd Conference on Computational Natural
  Language Learning}, pages 645--655, 2019.

\bibitem{mohri2018foundations}
Mehryar Mohri, Afshin Rostamizadeh, and Ameet Talwalkar.
\newblock {\em Foundations of machine learning}.
\newblock MIT press, 2018.

\bibitem{ni2017weakly}
Jian Ni, Georgiana Dinu, and Radu Florian.
\newblock Weakly supervised cross-lingual named entity recognition via
  effective annotation and representation projection.
\newblock In {\em Proceedings of the 55th Annual Meeting of the Association for
  Computational Linguistics (Volume 1: Long Papers)}, pages 1470--1480, 2017.

\bibitem{peng2019distantly}
Minlong Peng, Xiaoyu Xing, Qi~Zhang, Jinlan Fu, and Xuan-Jing Huang.
\newblock Distantly supervised named entity recognition using
  positive-unlabeled learning.
\newblock In {\em Proceedings of the 57th Annual Meeting of the Association for
  Computational Linguistics}, pages 2409--2419, 2019.

\bibitem{shalev2014understanding}
Shai Shalev-Shwartz and Shai Ben-David.
\newblock {\em Understanding machine learning: From theory to algorithms}.
\newblock Cambridge university press, 2014.

\bibitem{shang2018learning}
Jingbo Shang, Liyuan Liu, Xiaotao Gu, Xiang Ren, Teng Ren, and Jiawei Han.
\newblock Learning named entity tagger using domain-specific dictionary.
\newblock In {\em Proceedings of the 2018 Conference on Empirical Methods in
  Natural Language Processing}, pages 2054--2064, 2018.

\bibitem{si-etal-2023-santa}
Shuzheng Si, Zefan Cai, Shuang Zeng, Guoqiang Feng, Jiaxing Lin, and Baobao
  Chang.
\newblock {SANTA}: Separate strategies for inaccurate and incomplete annotation
  noise in distantly-supervised named entity recognition.
\newblock In {\em Findings of the Association for Computational Linguistics:
  ACL 2023}, pages 3883--3896, Toronto, Canada, July 2023. Association for
  Computational Linguistics.

\bibitem{tang2013recognizing}
Buzhou Tang, Hongxin Cao, Yonghui Wu, Min Jiang, and Hua Xu.
\newblock Recognizing clinical entities in hospital discharge summaries using
  structural support vector machines with word representation features.
\newblock In {\em BMC medical informatics and decision making}, volume~13,
  pages 1--10. BioMed Central, 2013.

\bibitem{touvron2023llama}
Hugo Touvron, Thibaut Lavril, Gautier Izacard, Xavier Martinet, Marie-Anne
  Lachaux, Timoth{\'e}e Lacroix, Baptiste Rozi{\`e}re, Naman Goyal, Eric
  Hambro, Faisal Azhar, et~al.
\newblock Llama: Open and efficient foundation language models.
\newblock {\em arXiv preprint arXiv:2302.13971}, 2023.

\bibitem{wang2023gptner}
Shuhe Wang, Xiaofei Sun, Xiaoya Li, Rongbin Ouyang, Fei Wu, Tianwei Zhang,
  Jiwei Li, and Guoyin Wang.
\newblock Gpt-ner: Named entity recognition via large language models.
\newblock 2023.

\bibitem{wang2021automated}
Xinyu Wang, Yong Jiang, Nguyen Bach, Tao Wang, Zhongqiang Huang, Fei Huang, and
  Kewei Tu.
\newblock Automated concatenation of embeddings for structured prediction.
\newblock In {\em Proceedings of the 59th Annual Meeting of the Association for
  Computational Linguistics and the 11th International Joint Conference on
  Natural Language Processing (Volume 1: Long Papers)}, pages 2643--2660, 2021.

\bibitem{wang2021improving}
Xinyu Wang, Yong Jiang, Nguyen Bach, Tao Wang, Zhongqiang Huang, Fei Huang, and
  Kewei Tu.
\newblock Improving named entity recognition by external context retrieving and
  cooperative learning.
\newblock In {\em Proceedings of the 59th Annual Meeting of the Association for
  Computational Linguistics and the 11th International Joint Conference on
  Natural Language Processing (Volume 1: Long Papers)}, pages 1800--1812, 2021.

\bibitem{xie2023selfimproving}
Tingyu Xie, Qi~Li, Yan Zhang, Zuozhu Liu, and Hongwei Wang.
\newblock Self-improving for zero-shot named entity recognition with large
  language models.
\newblock 2023.

\bibitem{xu-etal-2023-sampling}
Lu~Xu, Lidong Bing, and Wei Lu.
\newblock Sampling better negatives for distantly supervised named entity
  recognition.
\newblock In {\em Findings of the Association for Computational Linguistics:
  ACL 2023}, pages 4874--4882, Toronto, Canada, July 2023. Association for
  Computational Linguistics.

\bibitem{xu2017multi}
Yixing Xu, Chang Xu, Chao Xu, and Dacheng Tao.
\newblock Multi-positive and unlabeled learning.
\newblock In {\em Proceedings of the 26th International Joint Conference on
  Artificial Intelligence}, pages 3182--3188, 2017.

\bibitem{yang2018distantly}
Yaosheng Yang, Wenliang Chen, Zhenghua Li, Zhengqiu He, and Min Zhang.
\newblock Distantly supervised ner with partial annotation learning and
  reinforcement learning.
\newblock In {\em Proceedings of the 27th International Conference on
  Computational Linguistics}, pages 2159--2169, 2018.

\bibitem{yang1999re}
Yiming Yang and Xin Liu.
\newblock A re-examination of text categorization methods.
\newblock In {\em Proceedings of the 22nd annual international ACM SIGIR
  conference on Research and development in information retrieval}, pages
  42--49, 1999.

\bibitem{zhang2021biasing}
Wenkai Zhang, Hongyu Lin, Xianpei Han, and Le~Sun.
\newblock De-biasing distantly supervised named entity recognition via causal
  intervention.
\newblock In {\em Proceedings of the 59th Annual Meeting of the Association for
  Computational Linguistics and the 11th International Joint Conference on
  Natural Language Processing (Volume 1: Long Papers)}, pages 4803--4813, 2021.

\bibitem{zhao2023positive}
Zipei Zhao, Fengqian Pang, Yaou Liu, Zhiwen Liu, Chuyang Ye, et~al.
\newblock Positive-unlabeled learning for binary and multi-class cell detection
  in histopathology images with incomplete annotations.
\newblock {\em Machine Learning for Biomedical Imaging}, 1(December 2022
  issue):1--30, 2023.

\bibitem{zhou2022distantly}
Kang Zhou, Yuepei Li, and Qi~Li.
\newblock Distantly supervised named entity recognition via confidence-based
  multi-class positive and unlabeled learning.
\newblock In {\em Proceedings of the 60th Annual Meeting of the Association for
  Computational Linguistics (Volume 1: Long Papers)}, pages 7198--7211, 2022.

\end{thebibliography}

\newpage
\begin{appendices}

\section*{Appendix}
\label{sec:appendix}

In appendix, we first provide some necessary notations that will be used in the proof. Next, we show three Lemmas that will be used in the proof of Theorems \ref{theo} and \ref{theo2}. Finally, we showcase the proof of Theorems \ref{theo} and \ref{theo2}.

\subsection*{Basic Notations}
We summarize frequently used notations in the following list.
\begin{table}[h]
    \centering
    \caption{Notations.}
    \label{ch5_notation}
    \setlength{\extrarowheight}{4pt}
    \begin{tabular}{l r l r}
    \toprule
    Notation & Meaning & Notation & Meaning \\
    \midrule
    $R_{P_i}^{+}(f)$ & $\mathbb{E}_{\mathbf{x} \sim p(\mathbf{x}|y=i)} \ell(f(\mathbf{x}), i)$ & $\hat{R}^{+}_{P_i}(f)$ & $\frac{1}{n_{P_i}} \sum_{j=1}^{n_{P_i}} \ell (f(\mathbf{x}_{j}^{P_i}), i)$ \\
    $R_{P_i}^{-}(f)$ & $\mathbb{E}_{\mathbf{x} \sim p(\mathbf{x}|y=0)} \ell(f(\mathbf{x}), 0)$ & $\hat{R}^{-}_{P_i}(f)$ & $\frac{1}{n_{P_i}} \sum_{j=1}^{n_{P_i}} \ell (f(\mathbf{x}_{j}^{P_i}), 0)$ \\
    $R_{P}^{+}(f)$ & $\sum_{i=1}^{C}\pi_i R_{P_i}^{+}(f)$ & $\hat{R}_{P}^{+}(f)$ & $\sum_{i=1}^{C}\pi_i \hat{R}_{P_i}^{+}(f)$ \\
    $R_{P}^{-}(f)$ & $\sum_{i=1}^{C}\pi_i R_{P_i}^{-}(f)$ & $\hat{R}_{P}^{-}(f)$ & $\sum_{i=1}^{C}\pi_i \hat{R}_{P_i}^{-}(f)$ \\
    $R_{U}^{-}(f)$ & $\mathbb{E}_{\mathbf{x} \sim p(\mathbf{x})} \ell(f(\mathbf{x}), 0)$ & $\hat{R}_{U}^{-}(f)$ & $\frac{1}{n_{U}} \sum_{j=1}^{n_{U}} \ell (f(\mathbf{x}_{j}^{U}), 0)$ \\
    $R_{N}^{-}(f)$ & $R_{U}^{-}(f) - \sum_{i=1}^{C}\pi_i R_{P_i}^{-}(f)$ & $\hat{R}_{N}^{-}(f)$ & $\hat{R}_{U}^{-}(f) - \sum_{i=1}^{C}\pi_i \hat{R}_{P_i}^{-}(f)$ \\
    \bottomrule
    \end{tabular}
\end{table}

\section*{Lemmas}
\begin{Lemma}
    \label{lemma1}
    $\hat{R}_{MPU}(f)$ in Equation \eqref{mpu-estimator} is an unbiased estimator of $R(f)$ in Equation \eqref{mpu}.
\end{Lemma}

\begin{proof}
    \begin{equation*}
        \begin{aligned}
            \mathbb{E}[\hat{R}_{P_i}^{+}(f)] = & \mathbb{E}[\frac{1}{n_{P_i}} \sum_{j=1}^{n_{P_i}} \ell(f(\mathbf{x}_{j}^{P_i}), i)] \\
            = & \frac{1}{n_{P_i}} \sum_{j=1}^{n_{P_i}} \mathbb{E}[\ell(f(\mathbf{x}_{j}^{P_i}), i)] \\
            = & \frac{1}{n_{P_i}} \times n_{P_i} \mathbb{E}_{\mathbf{x}_{j}^{P_i} \sim p(\mathbf{x} | y=i)} [\ell(f(\mathbf{x}_{j}^{P_i}), i)] \\
            = & \frac{1}{n_{P_i}} \times n_{P_i} \mathbb{E}_{\mathbf{x} \sim p(\mathbf{x} | y=i)} [\ell(f(\mathbf{x}), i)] \\
            = & R_{P_i}^{+}(f).
        \end{aligned}
    \end{equation*}
    Similarly, we can derive that $\mathbb{E}[\hat{R}_{U}^{-}(f)]=R_{U}^{-}(f)$, and $\mathbb{E}[\hat{R}_{P_i}^{-}(f)]=R_{P_i}^{-}(f)$. Thus we have $\mathbb{E}[\hat{R}_{MPU}(f)]=R(f)$.
\end{proof}

\begin{Lemma}
    \label{lemma2}
    For any real numbers $x_1, x_2, x_3, x_4 \in \mathbb{R}$, the following inequality holds:
    $$|\max \left\{x_1, x_2\right\} - \max \left\{x_3, x_4\right\}| \leq |x_1 - x_3| + |x_2 - x_4|.$$
\end{Lemma}

\begin{proof}
    By triangular inequality, we have
    \begin{equation*}
        \begin{aligned}
            & |\max \left\{x_1, x_2\right\} - \max \left\{x_3, x_4\right\}| \\
            = & \left|\frac{x_1 + x_2 + |x_1 - x_2|}{2} - \frac{x_3 + x_4 + |x_3 - x_4|}{2} \right| \\
            = & \frac{1}{2} \left|(x_1 - x_3) + (x_2 - x_4) + (|x_1 - x_2| - |x_3 - x_4|) \right| \\
            \leq & \frac{1}{2} (|x_1 - x_3| + |x_2 - x_4| + ||x_1 - x_2| - |x_3 - x_4||) \\
            \leq & \frac{1}{2} (|x_1 - x_3| + |x_2 - x_4| + |(x_1 - x_2) - (x_3 - x_4)|) \\
            = & \frac{1}{2} (|x_1 - x_3| + |x_2 - x_4| + |(x_1 - x_3) - (x_2 - x_4)|) \\
            \leq & \frac{1}{2} (|x_1 - x_3| + |x_2 - x_4| + |x_1 - x_3| + |x_2 - x_4|) \\
            = & |x_1 - x_3| + |x_2 - x_4|.
        \end{aligned}
    \end{equation*}
\end{proof}

\begin{Lemma}
    \label{lemma3}
    The probability measure of $\mathcal{Q}^{-}(f)$ can be upper bounded as follows:
    \begin{equation}
        \label{bound-lemma}
        \text{Pr}(\mathcal{Q}^{-}(f)) \leq \exp\left(-\frac{2 (\alpha / C_{\ell})^2}{(1 + \lambda^2) \sum_{i=1}^{C} \pi_i^2 / n_{P_i} + 1 / n_{U}} \right).
    \end{equation}
\end{Lemma}

\begin{proof}
    We assume that there exists $\alpha > 0$, such that $R_{N}^{-}(f) - \lambda R_{P}^{+}(f) \geq \alpha$. For fixed $f$, partition all possible $\mathcal{X}$ into two disjoint sets
    \begin{equation*}
        \tiny
        \mathcal{Q}^{+}(f)=\left\{\mathcal{X} | \hat{R}_{N}^{-}(f) - \lambda \hat{R}_{P}^{+}(f) \geq 0 \right\},
    \end{equation*}
    and 
    \begin{equation*}
        \tiny
        \mathcal{Q}^{-}(f)=\left\{\mathcal{X} | \hat{R}_{N}^{-}(f) - \lambda \hat{R}_{P}^{+}(f) < 0 \right\}.
    \end{equation*}
    Let $F_{P_i}(\mathcal{X}_{P_i})$ be the cumulative distribution function of $\mathcal{X}_{P_i}$, $F_{U}(\mathcal{X}_{U})$ be that of $\mathcal{X}_{U}$, and
    \begin{equation*}
        F(\mathcal{X})=F(\mathcal{X}_{P_1}, \dots, \mathcal{X}_{P_C}, \mathcal{X}_{U})=\prod_{j=1}^{C} F_{P_j}(\mathcal{X}_{P_j}) F_{U}(\mathcal{X}_{U})
    \end{equation*}
    be the joint cumulative distribution function of $\mathcal{X}=(X_{P_1}, \dots, X_{P_C}, X_{U})$. Thus the probability measure of $\mathcal{Q}^{-}(f)$ is defined by
    $$\text{Pr}(\mathcal{Q}^{-}(f))=\int_{\mathcal{X} \in \mathcal{Q}^{-}(f)} dF(\mathcal{X}).$$
    Since $\hat{R}_{CMPU}(f)$ is identical to $\hat{R}_{MPU}(f)$ on $\mathcal{Q}^{+}(f)$ and different from $\hat{R}_{MPU}(f)$ on $\mathcal{Q}^{-}(f)$, we have $\text{Pr}(\mathcal{Q}^{-}(f))=\text{Pr}(\hat{R}_{CMPU}(f) \neq \hat{R}_{MPU}(f))$. By Lemma \ref{lemma1} we have
    \begin{equation*}
        \begin{aligned}
            & \mathbb{E}[\hat{R}_{CMPU}(f)] - R_{MPU}(f) = \mathbb{E}[\hat{R}_{CMPU}(f) - \hat{R}_{MPU}(f)] \\
            =& \int_{\mathcal{X} \in \mathcal{Q}^{+}(f)} [\hat{R}_{CMPU}(f) - \hat{R}_{MPU}(f)] dF(\mathcal{X}) + \int_{\mathcal{X} \in \mathcal{Q}^{-}(f)} [\hat{R}_{CMPU}(f) - \hat{R}_{MPU}(f)] dF(\mathcal{X}) \\
            =& \int_{\mathcal{X} \in \mathcal{Q}^{-}(f)} [\hat{R}_{CMPU}(f) - \hat{R}_{MPU}(f)] dF(\mathcal{X}).
        \end{aligned}
    \end{equation*}
    As a result, $\mathbb{E}[\hat{R}_{CMPU}(f)] - R_{MPU}(f) > 0$ if and only if $\int_{\mathcal{X} \in \mathcal{Q}^{-}(f)} [\hat{R}_{CMPU}(f) - R_{MPU}(f)] dF(\mathcal{X}) > 0$ due to the fact that $\hat{R}_{CMPU}(f) - \hat{R}_{MPU}(f) > 0$ on $\mathcal{Q}^{-}(f)$.

    We have assume that $0 \leq \ell(t, i) \leq C_{\ell}, 0 \leq i \leq C$, thus the change of $\hat{R}_{P_i}^{+}(f)$ will be no more than $C_{\ell} / n_{P_i}$ if some $x_{j}^{P_i} \in \mathcal{X}_{P_i}$ is replaced, and the change of $\hat{R}_{U}^{-}(f)$ will be no more than $C_{\ell} / n_{U}$ if some $x_{j}^{U} \in \mathcal{X}_{U}$ is replaced. Subsequently, by Lemma \ref{lemma1}, we have $\mathbb{E}[\hat{R}_{N}^{-}(f) - \lambda \hat{R}_{P}^{+}(f)]=R_{N}^{-}(f) - \lambda R_{P}^{+}(f)$, and McDiarmid's inequality ((D.16) in \cite{mohri2018foundations}),
    \begin{equation*}
        \begin{aligned}
            & \text{Pr}\left\{\lambda \hat{R}_{P}^{+}(f) - \hat{R}_{N}^{-}(f) - (\lambda R_{P}^{+}(f) - R_{N}^{-}(f)) \geq \alpha \right\} \\
            \leq & \exp\left(-\frac{2\alpha^2}{(1 + \lambda^2) C_{\ell}^2 \sum_{i=1}^{C} n_{P_i} (\frac{\pi_i}{n_{P_i}})^2 + n_{U} (\frac{C_{\ell}}{n_{U}})^2} \right) \\
            = & \exp\left(-\frac{2 (\alpha / C_{\ell})^2}{(1 + \lambda^2) \sum_{i=1}^{C} \pi_i^2 / n_{P_i} + 1 / n_{U}} \right).
        \end{aligned}
    \end{equation*}

    Taking into account that
    \begin{equation*}
        \begin{aligned}
            & \text{Pr}(\mathcal{Q}^{-}(f)) \\ 
            = & \text{Pr}(\hat{R}_{N}^{-}(f) - \lambda \hat{R}_{P}^{+}(f) < 0) \\
            \leq & \text{Pr}(\hat{R}_{N}^{-}(f) - \lambda \hat{R}_{P}^{+}(f) \leq R_{N}^{-}(f) - \lambda R_{P}^{+}(f) - \alpha) \\
            = & \text{Pr}(\hat{R}_{N}^{-}(f) - \lambda \hat{R}_{P}^{+}(f) - (R_{N}^{-}(f) - \lambda R_{P}^{+}(f)) \geq \alpha) \\
            \leq & \exp\left(-\frac{2 (\alpha / C_{\ell})^2}{(1 + \lambda^2) \sum_{i=1}^{C} \pi_i^2 / n_{P_i} + 1 / n_{U}} \right),
        \end{aligned}
    \end{equation*}
    the proof is completed. Denote the RHS of Equation \eqref{bound-lemma} as $\Delta_{f}$ later for convenience.
\end{proof}

\subsection*{Proof of Theorems \ref{theo} and \ref{theo2}}
\noindent \textbf{Theorem 1.}  For any $\delta > 0$, with probability at least $1 - \delta$, 
\begin{equation*}
		|\hat{R}_{CMPU}(f) - R(f)| \leq  C_{\delta} C_{\lambda} \left(\sum_{i=1}^{C} \frac{\pi_i}{\sqrt{n_{P_i}}} + \frac{1}{\sqrt{n_U}}\right) + \left((1 + \lambda) \sum_{i=1}^{C} \pi_i + 1 \right) C_{\ell} \Delta_{f},
\end{equation*}
where $$C_{\delta}=C_{\ell} \sqrt{\ln (2 / \delta) / 2},C_{\lambda}=\max\left\{2, 1 + \lambda \right\}.$$

\begin{proof}
    From the proof of Lemma \ref{lemma3}, we have known that
    \begin{equation*}
        \mathbb{E}[\hat{R}_{CMPU}(f)] - R_{MPU}(f) = \int_{\mathcal{X} \in \mathcal{Q}^{-}(f)} [\hat{R}_{CMPU}(f) - \hat{R}_{MPU}(f)] dF(\mathcal{X}).
    \end{equation*}
    Thus we have
    \begin{equation*}
        \begin{aligned}
            & \mathbb{E}[\hat{R}_{CMPU}(f)] - R_{MPU}(f) \\
             \leq & \sup_{\mathcal{X} \in \mathcal{Q}^{-}(f)} |\hat{R}_{CMPU}(f) - \hat{R}_{MPU}(f)|\int_{\mathcal{X} \in \mathcal{Q}^{-}(f)} dF(\mathcal{X}) \\
            = & \sup_{\mathcal{X} \in \mathcal{Q}^{-}(f)} |\lambda \hat{R}_{P}^{+}(f) - \hat{R}_{N}^{-}(f)| \text{Pr}(\mathcal{Q}^{-}(f)) \\
            \leq & [\lambda \hat{R}_{P}^{+}(f) + \hat{R}^{-}_{U}(f) + \hat{R}^{-}_{P}(f)] \text{Pr}(\mathcal{Q}^{-}(f)) \\
            \leq & (\lambda C_{\ell} \sum_{i=1}^{C} \pi_i + C_{\ell} + C_{\ell} \sum_{i=1}^{C} \pi_i) \text{Pr}(\mathcal{Q}^{-}(f)) \\
            \leq & \left((1 + \lambda) \sum_{i=1}^{C} \pi_i + 1 \right) C_{\ell} \Delta_{f}.
        \end{aligned}
    \end{equation*}
    
    Due to the fact that 
    \begin{equation*}
        \begin{aligned}
            & |\hat{R}_{CMPU}(f) - R_{MPU}(f)| \\
            = & \left|\hat{R}_{CMPU}(f) - \mathbb{E}[\hat{R}_{CMPU}(f)] + \mathbb{E}[\hat{R}_{CMPU}(f)] - R_{MPU}(f) \right| \\
            \leq & \left|\hat{R}_{CMPU}(f) - \mathbb{E}[\hat{R}_{CMPU}(f)] \right| + \left|\mathbb{E}[\hat{R}_{CMPU}(f)] - R_{MPU}(f) \right| \\
            \leq & \left|\hat{R}_{CMPU}(f) - \mathbb{E}[\hat{R}_{CMPU}(f)] \right| + \left((1 + \lambda) \sum_{i=1}^{C} \pi_i + 1 \right) C_{\ell} \Delta_{f},
        \end{aligned}
    \end{equation*}
    we only need to bound $|\hat{R}_{CMPU}(f) - \mathbb{E}[\hat{R}_{CMPU}(f)]|$. It can be known that the change of $\hat{R}_{CMPU}(f)$ will be no more than $\max\{1 + \lambda, 2\} C_{\ell} \pi_i / n_{P_i}$ if some $x_{j}^{P_i} \in \mathcal{X}_{P_i}$ is replaced, and the change of $\hat{R}_{CMPU}(f)$ will be no more than $C_{\ell} / n_{U}$ if some $x_{j}^{U} \in \mathcal{X}_{U}$ is replaced. Denote $C_{\lambda}=\max \left\{2, 1 + \lambda \right\}$ for convenience. By double-side McDiarmid's inequality ((D.16) and (D.17) in \cite{mohri2018foundations}), for any $\epsilon > 0$, we have
    \begin{equation*}
        \text{Pr}\left(|\hat{R}_{CMPU}(f) - \mathbb{E}[\hat{R}_{CMPU}(f)]| \geq \epsilon \right) \leq 2 \exp\left(\frac{-2 (\epsilon / C_{\ell})^2}{C_{\lambda}^2 \sum_{i=1}^{C} \pi_i^2 / n_{P_i} + 1 / n_{U}}\right).
    \end{equation*}
    In other words, let $C_{\delta}=C_{\ell}\sqrt{\ln(2 / \delta) / 2}$, with probability at least $1 - \delta$,
    as $n_{P_i}, n_{U} \to \infty$, the unconsistency decays exponentially. That is, 
    \begin{equation*}
		\begin{aligned}
		  & |\hat{R}_{CMPU}(f) - R(f)| \\ 
            \leq & C_{\delta} \sqrt{C_{\lambda}^2 \sum_{i=1}^{C} \pi_i^2 / n_{P_i} + 1 / n_{U}} + \left((1 + \lambda) \sum_{i=1}^{C} \pi_i + 1 \right) C_{\ell} \Delta_{f} \\
            \leq & C_{\delta} C_{\lambda} \left(\sum_{i=1}^{C} \frac{\pi_i}{\sqrt{n_{P_i}}} + \frac{1}{\sqrt{n_U}}\right) + \left((1 + \lambda) \sum_{i=1}^{C} \pi_i + 1 \right) C_{\ell} \Delta_{f},
		\end{aligned}
    \end{equation*}
    where the last inequality is derived by $\sqrt{\sum_{i=1}^{n} a_i} \leq \sum_{i=1}^{n} \sqrt{a_i}, a_i \geq 0$.
\end{proof}

\noindent \textbf{Theorem 2.}  Let $\hat{f}_{CMPU}=\arg \min_{f \in \mathcal{H}} \hat{R}_{CMPU}(f)$, $f^{*}=\arg \min_{f \in \mathcal{H}} R(f)$. Then, for any $\delta > 0$, with probability at least $1 - \delta$,
\begin{equation}
    \begin{aligned}
        R(\hat{f}_{CMPU}) - R(f^{*}) \leq & (16 + 8 \lambda) L_{\ell} \sum_{i=1}^{C} \pi_i \mathfrak{R}_{n_{P_i}, p_{P_i}}(\mathcal{H}) + 8L_{\ell} \mathfrak{R}_{n_{U}, p}(\mathcal{H})  \\
        & + 2 C_{\delta}^{'} C_{\lambda} \left(\sum_{i=1}^{C} \frac{\pi_i}{\sqrt{n_{P_i}}} + \frac{1}{\sqrt{n_U}}\right) 
        + 2 \left((1 + \lambda) \sum_{i=1}^{C} \pi_i + 1 \right) C_{\ell} \Delta_{f}.
    \end{aligned}
\end{equation}

\begin{proof}
    Define 
    $$\mathfrak{R}_{n, p}^{'}(\mathcal{H})=\mathbb{E}_{\mathcal{X}} \mathbb{E}_{\mathbf{\sigma}}\left[\sup_{f \in \mathcal{H}} \left|\frac{1}{n} \sum_{i=1}^{n} \sigma_i f(\mathbf{x}_i) \right|\right], \quad \mathfrak{R}_{n, p}(\mathcal{H})=\mathbb{E}_{\mathcal{X}} \mathbb{E}_{\mathbf{\sigma}}\left[\sup_{f \in \mathcal{H}} \frac{1}{n} \sum_{i=1}^{n} \sigma_i f(\mathbf{x}_i) \right].$$
    We first bound term $\sup_{f \in \mathcal{H}} |\hat{R}_{CMPU}(f) - R_{MPU}(f)|$.
    It is know that
    \begin{equation*}
        \begin{aligned}
            & \sup_{f \in \mathcal{H}} |\hat{R}_{CMPU}(f) - R_{MPU}(f)| \\ 
            \leq & \sup_{f \in \mathcal{H}} |\hat{R}_{CMPU}(f) - \mathbb{E}[\hat{R}_{MPU}(f)]| + \sup_{f \in \mathcal{H}} |\mathbb{E}[\hat{R}_{MPU}(f)] - \hat{R}_{CMPU}(f)| \\
            \leq & \sup_{f \in \mathcal{H}} |\hat{R}_{CMPU}(f) - \mathbb{E}[\hat{R}_{MPU}(f)]| + \left((1 + \lambda) \sum_{i=1}^{C} \pi_i + 1 \right) C_{\ell} \Delta_{f}.
        \end{aligned}
    \end{equation*}
    By single-side McDiamid's inequality, with probability at least $1 - \delta$,
    \begin{equation*}
        \begin{aligned}
            &\sup_{f \in \mathcal{H}} |\hat{R}_{CMPU}(f) - \mathbb{E}[\hat{R}_{MPU}(f)]| - \mathbb{E}\left[\sup_{f \in \mathcal{H}} |\hat{R}_{CMPU}(f) - \mathbb{E}[\hat{R}_{MPU}(f)]|\right] \\
            \leq & C_{\delta}^{'} C_{\lambda} \left(\sum_{i=1}^{C} \frac{\pi_i}{\sqrt{n_{P_i}}} + \frac{1}{\sqrt{n_{U}}}\right),
        \end{aligned}
    \end{equation*}
    where $C_{\delta}^{'}=C_{\ell} \sqrt{\ln(1/\delta) / 2}$. This indicates that we only need to bound $$\mathbb{E}\left[\sup_{f \in \mathcal{H}} |\hat{R}_{CMPU}(f) - \mathbb{E}[\hat{R}_{MPU}(f)]|\right]$$ to get the bound of $\sup_{f \in \mathcal{H}} |\hat{R}_{CMPU}(f) - \mathbb{E}[\hat{R}_{MPU}(f)]|$.

    Let $\mathcal{X}^{'}=(\mathcal{X}_{P}^{'}, \mathcal{X}_{U}^{'})$ be an independent copy of $\mathcal{X}=(\mathcal{X}_{P}, \mathcal{X}_{U})$, then we have 
    \begin{equation*}
        \label{bound-symm}
        \begin{aligned}
            & \mathbb{E}\left[\sup_{f \in \mathcal{H}} \left|\hat{R}_{CMPU}(f) - \mathbb{E}[\hat{R}_{MPU}(f)] \right| \right] \\
            = & \mathbb{E}_{(\mathcal{X}_{P}, \mathcal{X}_{U})} \left[\sup_{f \in \mathcal{H}} \left| \hat{R}_{CMPU}(f) - \mathbb{E}_{(\mathcal{X}_{P}^{'}, \mathcal{X}_{U}^{'})}[\hat{R}_{MPU}(f)] \right| \right] \\
            \overset{(i)}{=} & \mathbb{E}_{(\mathcal{X}_{P}, \mathcal{X}_{U})} \left[\sup_{f \in \mathcal{H}} \left| \mathbb{E}_{(\mathcal{X}_{P}^{'}, \mathcal{X}_{U}^{'})} [\hat{R}_{CMPU}(f) - \mathbb{E}_{(\mathcal{X}_{P}^{'}, \mathcal{X}_{U}^{'})}[\hat{R}_{MPU}(f)]] \right| \right] \\
            \overset{(ii)}{\leq} & \mathbb{E}_{(\mathcal{X}_{P}, \mathcal{X}_{U})} \left[\sup_{f \in \mathcal{H}} \mathbb{E}_{(\mathcal{X}_{P}^{'}, \mathcal{X}_{U}^{'})} \left|\hat{R}_{CMPU}(f) - \mathbb{E}_{(\mathcal{X}_{P}^{'}, \mathcal{X}_{U}^{'})}[\hat{R}_{MPU}(f)] \right|\right] \\
            \overset{(iii)}{\leq} & \mathbb{E}_{(\mathcal{X}_{P}, \mathcal{X}_{U})} \mathbb{E}_{(\mathcal{X}_{P}^{'}, \mathcal{X}_{U}^{'})}\left[\sup_{f \in \mathcal{H}} \left| \hat{R}_{CMPU}(f;\mathcal{X}_{P}, \mathcal{X}_{U}) - \hat{R}_{MPU}(f;\mathcal{X}_{P}^{'}, \mathcal{X}_{U}^{'}) \right| \right] \\
            \overset{(iv)}{=} & \mathbb{E}_{(\mathcal{X}_{P}, \mathcal{X}_{U}), (\mathcal{X}_{P}^{'}, \mathcal{X}_{U}^{'})}\left[\sup_{f \in \mathcal{H}} \left|\hat{R}_{CMPU}(f;\mathcal{X}_{P}, \mathcal{X}_{U}) - \hat{R}_{MPU}(f;\mathcal{X}_{P}^{'}, \mathcal{X}_{U}^{'}) \right|\right], \\
        \end{aligned}
    \end{equation*}
    where $\hat{R}_{CMPU}(f;\mathcal{X}_{P}, \mathcal{X}_{U})$ means that the risk $\hat{R}_{CMPU}(f)$ is calculated using random sample $(\mathcal{X}_{P}, \mathcal{X}_{U})$. Equation $(i)$ holds because $\hat{R}_{CMPU}(f)$ is calculated using random sample $(\mathcal{X}_{P}, \mathcal{X}_{U})$, the expectation under sample $(\mathcal{X}_{P}^{'}, \mathcal{X}_{U}^{'})$ does not change the value of $|\hat{R}_{CMPU}(f) - \mathbb{E}_{(\mathcal{X}_{P}^{'}, \mathcal{X}_{U}^{'})}[\hat{R}_{MPU}(f)]|$. Inequality $(ii)$ is derived by Jensen's inequality due to the convexity of absolute function $|\cdot|$. Inequality $(iii)$ is derived by Jensen's inequality due to the convexity of $\sup(\cdot)$. Equation $(iv)$ is due to the independence between $(\mathcal{X}_{P}, \mathcal{X}_{U})$ and $(\mathcal{X}_{P}^{'}, \mathcal{X}_{U}^{'})$, which means that the expectation under their marginal distribution is equivalent to the expectation under their joint distribution.

    Next we should bound $|\hat{R}_{CMPU}(f;\mathcal{X}_{P}, \mathcal{X}_{U}) - \hat{R}_{MPU}(f;\mathcal{X}_{P}^{'}, \mathcal{X}_{U}^{'})|$.
    By some simple calculation, we have
    \begin{equation*}
            \left|\hat{R}_{CMPU}(f;\mathcal{X}_{P}, \mathcal{X}_{U}) - \hat{R}_{MPU}(f;\mathcal{X}_{P}^{'}, \mathcal{X}_{U}^{'}) \right| = \left|\sum_{i=1}^{C} \pi_i (\hat{R}_{P_i}^{+}(f; \mathcal{X}_{P}) - \hat{R}_{P_i}^{+}(f; \mathcal{X}_{P}^{'})) + A - B \right|, 
    \end{equation*}
    where 
    \begin{equation*}
        \begin{aligned}
        A&=\max \left\{\lambda \sum_{i=1}^{C} \pi_i \hat{R}_{P_i}^{+}(f; \mathcal{X}_{P}), \hat{R}_{U}^{-}(f; \mathcal{X}_{U}) - \sum_{i=1}^{C} \pi_i \hat{R}_{P_i}^{-}(f; \mathcal{X}_{P}) \right\}, \\
        B&=\max \left\{\lambda \sum_{i=1}^{C} \pi_i \hat{R}_{P_i}^{+}(f; \mathcal{X}_{P}^{'}), \hat{R}_{U}^{-}(f; \mathcal{X}_{U}^{'}) - \sum_{i=1}^{C} \pi_i \hat{R}_{P_i}^{-}(f; \mathcal{X}_{P}^{'}) \right\}.
        \end{aligned}
    \end{equation*}
    By Lemma \ref{lemma2}, we have
    \begin{equation}
        \small
        \label{bound-lemma2}
        \begin{aligned}
            & \left|\sum_{i=1}^{C} \pi_i (\hat{R}_{P_i}^{+}(f; \mathcal{X}_{P}) - \hat{R}_{P_i}^{+}(f; \mathcal{X}_{P}^{'})) + A - B \right| \\
             \leq & \sum_{i=1}^{C} \pi_i \left| \hat{R}_{P_i}^{+}(f; \mathcal{X}_{P}) - \hat{R}_{P_i}^{+}(f; \mathcal{X}_{P}^{'})\right| + \lambda \sum_{i=1}^{C} \pi_i \left| \hat{R}_{P_i}^{+}(f; \mathcal{X}_{P}) - \hat{R}_{P_i}^{+}(f; \mathcal{X}_{P}^{'})\right| + \\ 
             & |\hat{R}_{U}^{-}(f; \mathcal{X}_{U}) - \hat{R}_{U}^{-}(f; \mathcal{X}_{U}^{'})| + \sum_{i=1}^{C} \pi_i \left| \hat{R}_{P_i}^{-}(f; \mathcal{X}_{P}) - \hat{R}_{P_i}^{-}(f; \mathcal{X}_{P}^{'})\right| \\
             = & (1 + \lambda) \sum_{i=1}^{C} \pi_i \left| \hat{R}_{P_i}^{+}(f; \mathcal{X}_{P}) - \hat{R}_{P_i}^{+}(f; \mathcal{X}_{P}^{'})\right| + |\hat{R}_{U}^{-}(f; \mathcal{X}_{U}) - \hat{R}_{U}^{-}(f; \mathcal{X}_{U}^{'})| + \sum_{i=1}^{C} \pi_i \left| \hat{R}_{P_i}^{-}(f; \mathcal{X}_{P}) - \hat{R}_{P_i}^{-}(f; \mathcal{X}_{P}^{'})\right|.
        \end{aligned}
    \end{equation}
    Combining Equations \eqref{bound-symm} and \eqref{bound-lemma2}, we have
    \begin{equation*}
        \begin{aligned}
            & \mathbb{E}\left[\sup_{f \in \mathcal{H}} |\hat{R}_{CMPU}(f) - \mathbb{E}[\hat{R}_{MPU}(f)]|\right] \\
            \leq & (1 + \lambda)\sum_{i=1}^{C} \pi_i \mathbb{E}_{(\mathcal{X}_{P}, \mathcal{X}_{P}^{'})} \left[\sup_{f \in \mathcal{H}} |\hat{R}_{P_i}^{+}(f; \mathcal{X}_{P}) - \hat{R}_{P_i}^{+}(f; \mathcal{X}_{P}^{'})|\right] + \\
            & \sum_{i=1}^{C} \pi_i \mathbb{E}_{(\mathcal{X}_{P}, \mathcal{X}_{P}^{'})} \left[\sup_{f \in \mathcal{H}} |\hat{R}_{P_i}^{-}(f; \mathcal{X}_{P}) - \hat{R}_{P_i}^{-}(f; \mathcal{X}_{P}^{'})|\right] + \\
            & \mathbb{E}_{(\mathcal{X}_{U}, \mathcal{X}_{U}^{'})} \left[\sup_{f \in \mathcal{H}} |\hat{R}_{U}^{-}(f; \mathcal{X}_{U}) - \hat{R}_{U}^{-}(f; \mathcal{X}_{U}^{'})|\right].
        \end{aligned}
    \end{equation*}

    Then we bound the expectation of each supremum value. Let $\tilde{\ell}(t, y) = \ell(t, y) - \ell(0, y), y \in \left\{0, \dots, C\right\}$. It is easy to verify that $\tilde{\ell}(t, y)$ is Lipschitz continuous with respect to $t$ with a Lipschitz constant $L_{\ell}$, and $\ell(t, y) - \ell(0, y)=\tilde{\ell}(t, y) - \tilde{\ell}(0, y)$.
    Notice that
    \begin{equation*}
        \begin{aligned}
            & \hat{R}_{P_i}^{+}(f;\mathcal{X}_{P}) - \hat{R}_{P_i}^{+}(f;\mathcal{X}_{P}^{'}) \\
            = & \frac{1}{n_{P_i}} \sum_{j=1}^{n_{P_i}} \ell (f(\mathbf{x}_{j}^{P_i}), i) - \frac{1}{n_{P_i}} \sum_{j=1}^{n_{P_i}} \ell (f(\mathbf{x'}_{j}^{P_i}), i) \\
            = & \frac{1}{n_{P_i}} \sum_{j=1}^{n_{P_i}} \left(\ell (f(\mathbf{x}_{j}^{P_i}), i) -  \ell (f(\mathbf{x'}_{j}^{P_i}), i)\right) \\
            = & \frac{1}{n_{P_i}} \sum_{j=1}^{n_{P_i}} \left(\tilde{\ell} (f(\mathbf{x}_{j}^{P_i}), i) -  \tilde{\ell}(f(\mathbf{x'}_{j}^{P_i}), i)\right),
        \end{aligned}
    \end{equation*}
    thus it is already a standard form where we can attach Rademacher variables to every $\tilde{\ell} (f(\mathbf{x}_{j}^{P_i}), i) -  \tilde{\ell}(f(\mathbf{x'}_{j}^{P_i}), i)$. By formula (3.8)-(3.13) in \cite{mohri2018foundations}, we have
    \begin{equation*}
        \mathbb{E}_{(\mathcal{X}_{P}, \mathcal{X}_{P}^{'})} \left[\sup_{f \in \mathcal{H}} |\hat{R}_{P_i}^{+}(f;\mathcal{X}_{P}) - \hat{R}_{P_i}^{+}(f;\mathcal{X}_{P}^{'})|\right] \leq 2 \mathfrak{R}^{'}_{n_{P_i}, p_{P_i}}(\tilde{\ell}(\cdot, i) \circ \mathcal{H}),
    \end{equation*}
    where $\tilde{\ell}(\cdot, i) \circ \mathcal{H}=\left\{\tilde{\ell}(\cdot, i) \circ f | f \in \mathcal{H} \right\}$ is the composite function class of $\tilde{\ell}(\cdot, i)$ and $\mathcal{H}$. The rest two expectations can be bounded in the same way, that is, 
    \begin{equation}
        \label{cmpu-expectation}
        \begin{aligned}
            & \mathbb{E}\left[\sup_{f \in \mathcal{H}} |\hat{R}_{CMPU}(f) - \mathbb{E}[\hat{R}_{MPU}(f)]|\right] \\
            \leq & 2(1 + \lambda) \sum_{i=1}^{C} \pi_i \mathfrak{R}^{'}_{n_{P_i}, p_{P_i}}(\tilde{\ell}(\cdot, i) \circ \mathcal{H}) + 2 \sum_{i=1}^{C} \pi_i \mathfrak{R}^{'}_{n_{P_i}, p_{P_i}}(\tilde{\ell}(\cdot, 0) \circ \mathcal{H}) + 2 \mathfrak{R}^{'}_{n_{U}, p}(\tilde{\ell}(\cdot, 0) \circ \mathcal{H}),
        \end{aligned}
    \end{equation}
    among which 
    $$\mathfrak{R}^{'}_{n_{P_i}, p_{P_i}}(\tilde{\ell}(\cdot, i) \circ \mathcal{H}) \leq 2 L_{\ell} \mathfrak{R}^{'}_{n_{P_i}, p_{P_i}}(\mathcal{H})=2 L_{\ell} \mathfrak{R}_{n_{P_i}, p_{P_i}}(\mathcal{H})$$ 
    by Talagrand's contradiction Lemma \citep{ledoux1991probability} and the assumption that $\mathcal{H}$ is close negation. Similarly, we have
    $$\mathfrak{R}^{'}_{n_{P_i}, p_{P_i}}(\tilde{\ell}(\cdot, 0) \circ \mathcal{H}) \leq 2 L_{\ell} \mathfrak{R}_{n_{P_i}, p_{P_i}}(\mathcal{H}),$$
    $$\mathfrak{R}^{'}_{n_{U}, p}(\tilde{\ell}(\cdot, 0) \circ \mathcal{H}) \leq 2 L_{\ell} \mathfrak{R}_{n_{U}, p}(\mathcal{H}).$$
    Furthermore, \eqref{cmpu-expectation} can be bounded by
    \begin{equation*}
        \begin{aligned}
            & \mathbb{E}\left[\sup_{f \in \mathcal{H}} |\hat{R}_{CMPU}(f) - \mathbb{E}[\hat{R}_{MPU}(f)]|\right] \\
            \leq & 2(1 + \lambda) \sum_{i=1}^{C} \pi_i \mathfrak{R}^{'}_{n_{P_i}, p_{P_i}}(\tilde{\ell}(\cdot, i) \circ \mathcal{H}) + 2 \sum_{i=1}^{C} \pi_i \mathfrak{R}^{'}_{n_{P_i}, p_{P_i}}(\tilde{\ell}(\cdot, 0) \circ \mathcal{H}) + 2 \mathfrak{R}^{'}_{n_{U}, p}(\tilde{\ell}(\cdot, 0) \circ \mathcal{H}) \\
            \leq & (4 + 4 \lambda) \sum_{i=1}^{C} \pi_i \mathfrak{R}_{n_{P_i}, p_{P_i}}(\mathcal{H}) + 4 \sum_{i=1}^{C} \pi_i \mathfrak{R}_{n_{P_i}, p_{P_i}}(\mathcal{H}) + 4 \mathfrak{R}_{n_{U}, p}(\mathcal{H}) \\
            = & (8 + 4 \lambda) \sum_{i=1}^{C} \pi_i \mathfrak{R}_{n_{P_i}, p_{P_i}}(\mathcal{H}) + 4 \mathfrak{R}_{n_{U}, p}(\mathcal{H}).
        \end{aligned}
    \end{equation*}

    Let $\hat{f}_{CMPU}=\arg \min_{f \in \mathcal{H}} \hat{R}_{CMPU}(f)$, $f^{*}=\arg \min_{f \in \mathcal{H}} R(f)$, we have
    \begin{equation*}
        \begin{aligned}
            & R(\hat{f}_{CMPU}) - R(f^{*}) \\ 
            = & \underbrace{(\hat{R}_{CMPU}(\hat{f}_{CMPU}) - \hat{R}_{CMPU}(f^{*}))}_{a} + \underbrace{(R(\hat{f}_{CMPU}) - \hat{R}_{CMPU}(\hat{f}_{CMPU}))}_{b} + \underbrace{(\hat{R}_{CMPU}(f^{*}) - R(f^{*}))}_{c} \\
        \end{aligned}
    \end{equation*}
    Hence $\hat{f}_{CMPU}$ is the minimizer of $\hat{R}_{CMPU}(f)$, $a < 0$. For $b$ and $c$, they both have the form of $\hat{R}_{CMPU}(\cdot) - R(\cdot)$, thus both $b$ and $c$ can be upper bounded by $\sup_{f \in \mathcal{H}} |\hat{R}_{CMPU}(f) - R(f)|$. Then we obtain
    \begin{equation*}
        \begin{aligned}
            & R(\hat{f}_{CMPU}) - R(f^{*}) \leq 0 + 2 \sup_{f \in \mathcal{H}} |\hat{R}_{CMPU}(f) - R(f)| \\
            \leq & (16 + 8 \lambda) L_{\ell} \sum_{i=1}^{C} \pi_i \mathfrak{R}_{n_{P_i}, p_{P_i}}(\mathcal{H}) + 8L_{\ell} \mathfrak{R}_{n_{U}, p}(\mathcal{H}) + \\ 
            & 2 C_{\delta}^{'} C_{\lambda} \left(\sum_{i=1}^{C} \frac{\pi_i}{\sqrt{n_{P_i}}} + \frac{1}{\sqrt{n_U}}\right) + 2 \left((1 + \lambda) \sum_{i=1}^{C} \pi_i + 1 \right) C_{\ell} \Delta_{f}.
        \end{aligned}
    \end{equation*}
\end{proof}

\end{appendices}


\end{document}